\documentclass[sigconf]{acmart}
\renewcommand\footnotetextcopyrightpermission[1]{} % removes footnote with conference information in first column
\makeatletter
\renewcommand\@formatdoi[1]{\ignorespaces}
\makeatother
\usepackage{amsfonts}
\usepackage{array, makecell}
\usepackage{mathtools}
\usepackage{hyperref}
\usepackage{graphicx}
\usepackage{subfigure}
\usepackage{enumerate}
\usepackage{amsmath}
\usepackage{color}
\usepackage{amsthm}
\usepackage{algorithm}
\usepackage[noend]{algpseudocode}
\usepackage{stfloats}
\usepackage{multirow}
\theoremstyle{definition}
\usepackage{color}
\allowdisplaybreaks[4]

\newcommand{\bp}{\begin{proof} \small }
\newcommand{\ep}{\end{proof} \normalsize}
\newcommand{\epx}{\end{proof} \small}
\newcommand{\bpa}{\begin{proofappx} \footnotesize }
\newcommand{\epa}{\end{proofappx} \small }
\newtheorem{theorem}{Theorem}
\newtheorem{proposition}{Proposition}

\newtheorem{assumption}{Assumption}
\newtheorem{definition}{Definition}

\newtheorem*{theorem*}{Theorem}
\newtheorem*{proposition*}{Proposition}
\newtheorem*{corollary*}{Corollary}
\newtheorem*{lemma*}{Lemma}
\newtheorem*{assumption*}{Assumption}
\newtheorem*{definition*}{Definition}
\newtheorem*{claim*}{Claim}

\newcommand{\be}{\begin{equation}}
\newcommand{\ee}{\end{equation}}
\newcommand{\bs}{\begin{subequations}}
\newcommand{\es}{\end{subequations}}
\newcommand{\bq}{\begin{eqnarray}}
\newcommand{\eq}{\end{eqnarray}}
\newcommand{\bqn}{\begin{eqnarray*}}
\newcommand{\eqn}{\end{eqnarray*}}

\newcommand{\ba}{\left[ \begin{array}}
\newcommand{\ea}{\\ \end{array} \right]}
\newcommand{\ben}{\begin{enumerate}}
\newcommand{\een}{\end{enumerate}}
\def\real{{\mathchoice%
{\hbox{\rm\setbox1=\hbox{I}\copy1\kern-.45\wd1 R}}
{\hbox{\rm\setbox1=\hbox{I}\copy1\kern-.45\wd1 R}}
{\hbox{\scriptsize\rm\setbox1=\hbox{I}\copy1\kern-.45\wd1 R}}
{\hbox{\scriptsize\rm\setbox1=\hbox{I}\copy1\kern-.45\wd1 R}}}}

\def\Zint{{\mathchoice{\setbox1=\hbox{\sf Z}\copy1\kern-.75\wd1\box1}
{\setbox1=\hbox{\sf Z}\copy1\kern-.75\wd1\box1}
{\setbox1=\hbox{\scriptsize\sf Z}\copy1\kern-.75\wd1\box1}
{\setbox1=\hbox{\scriptsize\sf Z}\copy1\kern-.75\wd1\box1}}}
\newcommand{\complex}{ \hbox{\rm C\kern-0.45em\rule[.07em]{.02em}{.58em}%
\kern 0.43em}}

\makeatletter
\newcommand{\algmargin}{\the\ALG@thistlm}
\makeatother
\newlength{\whilewidth}
\algdef{SE}[parWHILE]{parWhile}{EndparWhile}[1]
{\parbox[t]{\dimexpr\linewidth-\algmargin}{%
		\hangindent\whilewidth\strut\algorithmicwhile\ #1\ \algorithmicdo\strut}}{\algorithmicend\ \algorithmicwhile}%
\algnewcommand{\parState}[1]{\State%
	\parbox[t]{\dimexpr\linewidth-\algmargin}{\strut #1\strut}}

\ifodd 1

\else

\fi

\begin{document}

\title{CAFE: Carbon-Aware Federated Learning in Geographically Distributed Data Centers}
% \titlenote{This work is partially supported by NSF under grants 2006630, 2033681, 2029858, 2044991 and 2319780.}

\author{Jieming Bian}
\email{jxb1974@miami.edu}
\affiliation{%
  \institution{University of Miami}
  \city{Coral Gables}
  \state{FL}
  \country{USA}
}

\author{Lei Wang}
\email{lxw725@miami.edu}
\affiliation{%
  \institution{University of Miami}
  \city{Coral Gables}
  \state{FL}
  \country{USA}
}

\author{Shaolei Ren}
\email{sren@ece.ucr.edu}
\affiliation{%
  \institution{University of California, Riverside}
  \city{Riverside}
  \state{CA}
  \country{USA}
}

\author{Jie Xu}
\email{jiexu@miami.edu}
\affiliation{%
  \institution{University of Miami}
  \city{Coral Gables}
  \state{FL}
  \country{USA}
}

% \authornote{This work is partially supported by NSF under grants 2006630, 2033681, 2029858, 2044991 and 2319780.}


\begin{abstract}
Training large-scale artificial intelligence (AI) models demands significant computational power and energy, leading to increased carbon footprint with potential environmental repercussions. This paper delves into the challenges of training AI models across geographically distributed (geo-distributed) data centers, emphasizing the balance between learning performance and carbon footprint. We consider Federated Learning (FL) as a solution, which prioritizes model parameter exchange over raw data, ensuring data privacy and compliance with local regulations. Given the variability in carbon intensity across regions, we propose a new framework called CAFE (short for Carbon-Aware Federated Learning) to optimize training within a fixed carbon footprint budget. Our approach incorporates coreset selection to assess learning performance, employs the Lyapunov drift-plus-penalty framework to address the unpredictability of future carbon intensity, and devises an efficient algorithm to address the combinatorial complexity of the data center selection. Through extensive simulations using real-world carbon intensity data, we demonstrate the efficacy of our algorithm, highlighting its superiority over existing methods in optimizing learning performance while minimizing environmental impact.
\end{abstract}






\maketitle
\pagestyle{plain}
\section{Introduction}

Leveraging the advancements in deep neural networks, artificial intelligence (AI) has evolved into an indispensable juggernaut, driving scientific breakthroughs, propelling business expansion, and confronting global challenges across various critical domains \cite{heaton2018ian, rolnick2022tackling, kang2023optical}. The efficacy of AI hinges significantly on computationally intensive calculations, essential for extracting valuable insights from vast datasets. This training process demands both robust computational capabilities and extensive data, posing a critical challenge. To surmount this challenge, current AI models, particularly large generative models like GPT-3 \cite{brown2020language}, are trained within data centers. These data centers are equipped with expansive clusters of servers, each hosting multiple graphics processing units (GPUs), ensuring the requisite computational power for training large models.

% \begin{figure}[t]
% \centering
% \includegraphics[width=0.8\linewidth]{figures/model_size.pdf}
% \caption{The size of current AI models is experiencing exponential growth.}
% \label{fig:LLM}
% \end{figure}

\begin{figure}[t]
    \centering
    \begin{minipage}[t]{0.48\linewidth}
       \includegraphics[width=1\linewidth, height=0.65\linewidth]{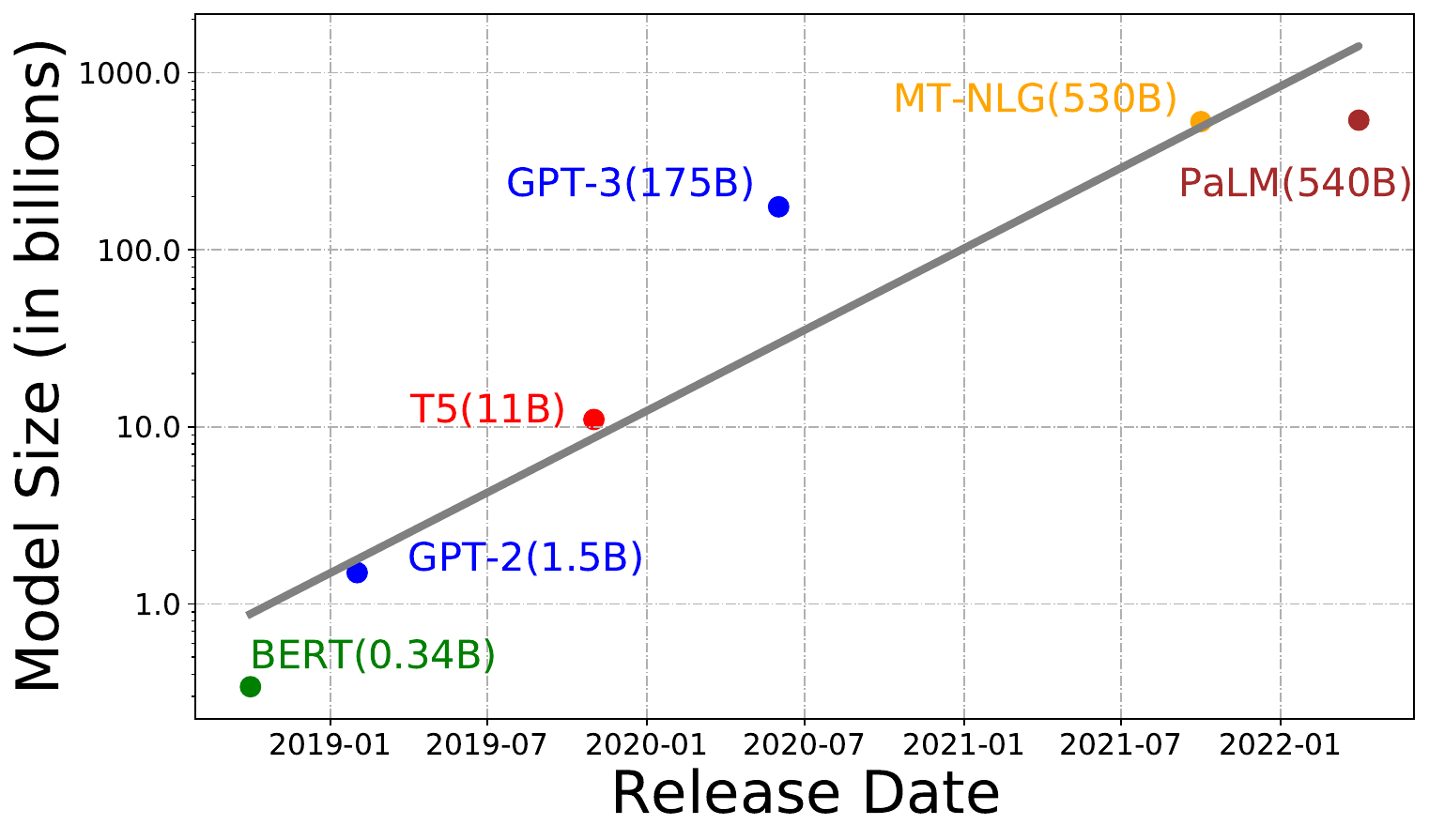}
        \caption{The size of current AI models is experiencing exponential growth.}
        \label{fig:LLM}
    \end{minipage}\hfill
    \begin{minipage}[t]{0.48\linewidth}
        \includegraphics[width=1\linewidth, height=0.65\linewidth]{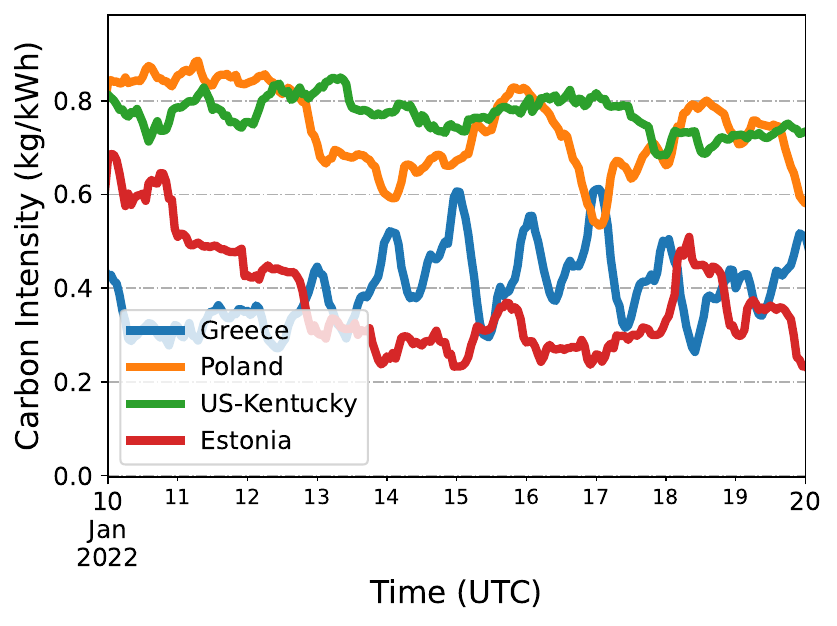}
        \caption[5pt]{The carbon intensity varies across different times and data centers.}
        \label{fig:Intensity}
    \end{minipage}
    \vspace{-20 pt}
\end{figure}

While large-scale AI models offer remarkable advantages, they are accompanied by substantial costs. The exponential growth in the size of current AI models (i.e. shown in Fig. \ref{fig:LLM}) has created an insatiable energy demand, leading to adverse environmental consequences \cite{brown2020language, patterson2022carbon, schwartz2020green, strubell2019energy, henderson2020towards, luccioni2022estimating}. For instance, aside from the environmental toll of chip manufacturing, involving raw material extraction and the use of toxic chemicals \cite{switzer2023junkyard, gupta2022act}, and the noise pollution generated by operating AI servers \cite{monserrate2022cloud}, training a large language model like GPT-3 \cite{brown2020language} and LaMDA \cite{thoppilan2022lamda} can easily devour hundreds of megawatt-hours of electricity. Even with the adoption of industry best practices to mitigate AI's resource consumption, AI models at Google, for example, now account for approximately 10-15\% of its data center's total energy consumption \cite{patterson2022carbon}. Such substantial energy consumption during AI model training results in the generation of numerous tons of carbon emissions. These incremental carbon emissions have led to increased social costs \cite{cruz2022local} and may contribute to local issues such as ozone pollution and premature mortality \cite{jacobson2010enhancement}. With the prediction that both AI models and AI data will continue to expand \cite{wu2022sustainable}, potentially exacerbating environmental issues like increasing carbon emissions, there has been a recent surge in research focused on making AI more energy-efficient and environmentally friendly. Research endeavors have explored various approaches, including computationally efficient training and inference techniques \cite{chen2023frugalgpt, rajbhandari2022deepspeed}, the design of energy-efficient GPUs and accelerators \cite{xu2020autodnnchip, gupta2020accelerator}, the implementation of green cloud infrastructures \cite{gandhi2022metrics, bashir2021enabling, acun2023carbon}, and the adoption of carbon-aware scheduling practices \cite{wu2022sustainable, henderson2020towards}.

%While current studies show promise in reducing carbon emissions, they often overlook the critical aspect of AI model performance. 
The effectiveness of training large models heavily relies on access to extensive training data \cite{brown2020language}. Typically, data centers gather data from the country or region in which they are located. These data centers, such as Google's data centers \cite{patterson2022carbon}, are distributed across various countries. Since data preferences vary between countries, the data collected by data centers in different countries may exhibit significant non-independence and non-identical distribution (non-IID). Additionally, these data centers are subject to the laws and regulations of their respective host countries, which often prohibit the sharing of raw data across international borders. Consequently, data collected by one data center can only be used for training purposes within that specific data center. Therefore, training a large AI model solely within one data center is unlikely to yield satisfactory learning performance due to data bias. 

Considering the challenges of managing non-IID data distribution and the constraints on transferring raw data between centers, Federated Learning \cite{mcmahan2017communication} (FL) emerges as an encouraging solution for training AI models across geographically distributed (geo-distributed) data centers. Unlike conventional distributed learning approaches, FL prioritizes the exchange of model parameters over raw data, ensuring data privacy and compliance with local regulations. Additionally, training across such geo-distributed data centers using FL has the potential to reduce overall carbon emissions. Due to geographical disparities, the environmental impact of AI model training varies significantly by region. For example, regional disparities in carbon intensity are evident; in 2020, only 4\% of the energy consumed by Google's data center in Singapore came from carbon-free sources, compared to 94\% in Finland \cite{patterson2022carbon}, resulting in a substantial 23-fold difference. Fig. \ref{fig:Intensity} illustrates the variability in carbon intensity over both time and among different data centers. By strategically selecting data centers located in regions with lower carbon intensity for training, we can effectively curtail the total carbon emissions associated with AI model training. Nonetheless, effective data center selection presents several technical challenges outlined below:

% \begin{figure}[h]
% \centering
% \includegraphics[width=0.7\linewidth]{figures/Intensity.pdf}
% \caption{The carbon intensity varies across different times and data centers.}
% \label{fig:Intensity}
% \end{figure}

\textbf{Balancing Learning Performance and Carbon Footprint:} Given the imperative to reduce overall carbon emissions, it is not feasible to have all data centers actively engaged throughout the entire training process. Consequently, within each training round of FL, a specific subset of data centers is chosen to participate. Notably, the selection of diverse subsets of data centers during training can lead to varied learning performance outcomes due to variations in their data distributions and quality. However, this dependency remains elusive, as learning performance can only be assessed for the chosen subset after the training process is completed.

\textbf{Unpredictable Future Carbon Intensity:} The carbon emission budget is a crucial long-term consideration, spanning the entire training process. A substantial carbon footprint at present implies that a diminished carbon emission budget remains for forthcoming training rounds. Consequently, decisions regarding data center selection are intrinsically interconnected over time. Nevertheless, due to the formidable challenge of forecasting future carbon intensity for estimating potential emissions, data center choices must rely solely on presently accessible information. While \cite{10.1145/3538637.3538849, 10.1145/3575813.3597346, 10.1145/3607114.3607117} enable the prediction of carbon intensity in the short-term future, the estimation of long-term future carbon intensity remains elusive, necessitating the problem to be treated as an online problem, albeit with a shifted scope. 

\textbf{The Intricate Combinatorial Nature of Selection Decision-Making:} It is evident that the task of choosing the ideal set of data centers to engage in FL constitutes a combinatorial optimization challenge. While an exhaustive search can theoretically yield the optimal solution within a finite timeframe, the complexity of this process escalates exponentially with the total number of data centers involved, rendering such an approach impractical in real-world scenarios. Hence, the development of a low-complexity algorithm becomes imperative.

In this paper, we propose a novel framework called \underline{C}arbon \underline{A}ware \underline{F}ederated L\underline{e}arning (CAFE).  Our primary focus centers on developing an adaptive data center selection algorithm, in scenarios where a fixed carbon footprint budget must be maintained throughout an extended training period. To tackle the challenges outlined above, our algorithm introduces the following key innovations.

1. In order to assess the learning performance of the chosen data centers, we utilize the notion of coreset selection to enhance the efficiency of machine learning training. We create a utility function designed to gauge how effectively a selected subset of data centers can represent the entire dataset when aggregated by the server. To tackle the challenge of evaluating learning performance before the training process concludes, we introduce a probing step. This step is designed to estimate the current model gradients for all data centers, leveraging small, randomly selected data samples. This approach not only facilitates utility calculation but also minimizes associated overhead.

2. In order to conduct efficient data center selection in the absence of future carbon intensity information, we employ the Lyapunov drift-plus-penalty framework to decompose the long-term optimization into individual per-slot problems. Our algorithm comes with theoretical performance guarantees, establishing a lower bound on the achieved learning utility and an upper bound on the resulting carbon footprint. This analysis not only extends previous theoretical findings but also acknowledges that the learning objective function is influenced by past selection decisions.

3. To tackle the intricate combinatorial complexity inherent in the data center selection problem, we cast it as a constrained submodular maximization challenge. We then devise efficient greedy algorithms, with performance guarantees, to address this problem. It is worth noting that these submodular problems are inherently non-monotone, due to the carbon footprint budget constraint. This non-monotonicity distinguishes them from conventional approaches used in similar FL client selection problems, rendering those approaches ineffective.

To assess the effectiveness of the proposed algorithm, we conduct comprehensive experimental simulations. These simulations involve scaling up small model training to replicate the complexities encountered when training large models. We utilize real-world carbon intensity data from geo-distributed data centers for these simulations. Our extensive experimentation demonstrates the efficiency and superiority of our approach, illustrating its substantial outperformance in comparison to baseline methods.

\section{Related Works}

Sustainable AI has garnered substantial attention in recent years \cite{brown2020language, henderson2020towards, thoppilan2022lamda, patterson2022carbon}. In the pursuit of rendering AI more energy-efficient and environmentally sustainable, many strategies and studies have been explored. These include endeavors to enhance computational efficiency in training and inference \cite{chen2023frugalgpt, rajbhandari2022deepspeed}, the design of energy-efficient GPUs and accelerators \cite{xu2020autodnnchip, gupta2020accelerator}, the implementation of carbon-aware task scheduling \cite{wu2022sustainable, henderson2020towards}, and the development of green cloud infrastructures \cite{gandhi2022metrics, bashir2021enabling, acun2023carbon}, among others. Most of these efforts primarily revolve around optimizing AI training within a single data center. However, recognizing the advantages stemming from data accumulation and the potential for reducing carbon emissions, we delve into the collaborative training of AI models across multiple geographically distributed data centers. In addition to these computational approaches, there exist non-computational strategies aimed at improving AI's environmental sustainability. For instance, data center operators have increasingly embraced carbon-free energy sources like solar and wind power to reduce AI's carbon footprint \cite{wu2022sustainable, google2022environmental, meta2021sustainability}. It is important to note that these non-computational approaches are complementary to our work.

Most existing FL studies \cite{mcmahan2017communication, lim2020federated, kang2020reliable, nguyen2021federated, li2021hermes} focus on enhancing model performance for mobile and edge devices. However, there is an increasing interest in exploring application domains involving clients with significant computational capabilities and electricity demands \cite{rieke2020future, so2022fedspace, nguyen2022deep}. Our research addresses data center-level applications using FL, presenting unique challenges in balancing learning performance and carbon footprint. Furthermore, our study pertains to the client selection problem in FL. Several selection criteria have been explored in recent literature, such as sampling clients based on their local dataset size \cite{mcmahan2017communication}, prioritizing clients with larger update norms \cite{chen2020optimal}, and selecting clients with higher losses \cite{cho2020client}. These methods differ from our approach, which aims to select a `coreset' representative of the entire client set. The most closely related work in FL considering client selection is \cite{balakrishnan2022diverse}, which also emphasizes the `coreset' in an FL context. However, \cite{balakrishnan2022diverse} exclusively emphasizes learning performance, using staleness information as a metric. In contrast, our research investigates the trade-off between learning performance, measured through a probing step, and carbon footprint, constrained by a long-term carbon footprint budget, thereby introducing additional complexity to the problem. Given the carbon footprint budget constraint, our selection problem's submodular nature is inherently non-monotone, setting our method apart from \cite{balakrishnan2022diverse}. Moreover, since we lack information on future carbon intensity, our study addresses a long-term online optimization problem, utilizing the Lyapunov drift-plus-penalty framework to break down the long-term optimization into individual per-slot challenges. Unlike other studies that apply the Lyapunov framework \cite{neely2022stochastic, sun2017survey}, our approach contends with the challenges posed by our learning objective function being influenced by previous selection decisions.

Our research is related with the existing literature on task scheduling in geo-distributed data centers \cite{hu2017time, li2023towards, polverini2013thermal, convolbo2018geodis}. For instance, \cite{hu2017time} aims to optimize big data processing across such centers, targeting reductions in both completion times and network costs. \cite{li2023towards} delves into the carbon and water footprints associated with AI model inference tasks. \cite{polverini2013thermal} introduces an online scheduling algorithm designed to balance energy costs, fairness, and latency, all while averting server overheating. Meanwhile, \cite{convolbo2018geodis} emphasizes optimizing data locality and transfer to diminish job makespan. While much of the current literature on geo-distributed data center scheduling primarily addresses resource allocation and distribution, our work uniquely concentrates on the training process and performance of large AI models within these centers.

\section{System Model}
We consider a training task that is distributed across a set of geo-distributed data centers, denoted as $\mathcal{N} = \{1,2,\dots, N\}$. Each participating data center, represented as $i\in \mathcal{N}$, houses a local dataset $\mathcal{D}_i$. These local datasets $\mathcal{D}_i$ are gathered by regional customers (for instance, data centers in different countries collect data specific to their respective countries), hence each local dataset $\mathcal{D}_i$ is non-IID. The goal of these data centers is to carry out a supervised learning task. In this context, the local dataset $\mathcal{D}_i$ is defined as a collection of data samples, each represented as a set of input-output pairs $\{x_j, y_j\}_{j=1}^{D_i}$. Here, $x_j \in \mathbb{R}^d$ is a $d$-dimensional input feature vector, $y_j \in \mathbb{R}$ is the corresponding output label, and $D_i$ represents the set of local dataset $\mathcal{D}_i$.

\begin{figure*}[t]
\vspace{-10 pt}
\centering
\includegraphics[width=0.95\linewidth]{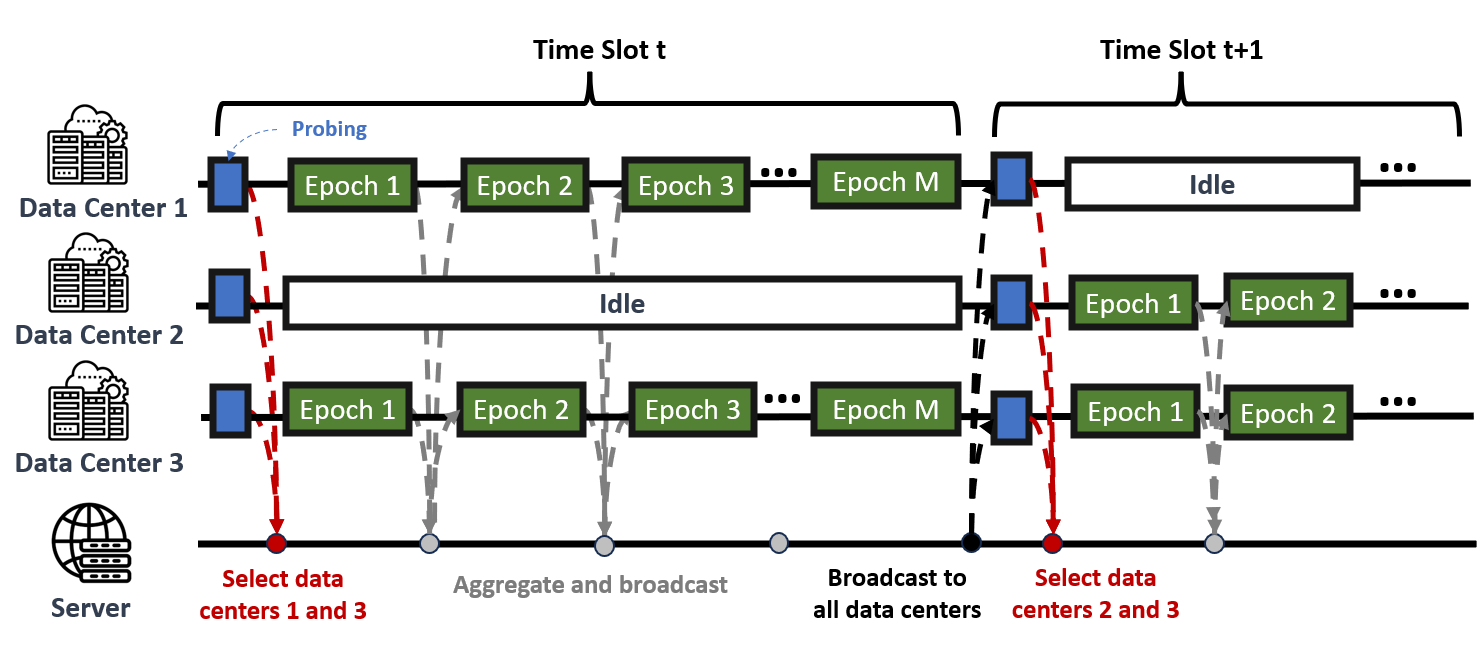}
\vspace{-10 pt}
\caption{The overview of CAFE.}
\label{fig:Cafe}
\vspace{-5 pt}
\end{figure*}

We briefly introduce the training process of our proposed Carbon-Aware Federated Learning (CAFE) as depicted in Fig. \ref{fig:Cafe}. The training process is organized into $T$ training time slots. For simplicity and without loss of generality, we assume that each data center possesses similar computational capabilities. Consequently, during each time slot, every selected data center can execute $M$ training epochs. In each training time slot, a server, which can be either a dedicated server or one of the data centers, aggregates local updates from all data centers. During a training time slot, denoted as $t \in \{1,\dots, T\}$, the data centers carry out the following sequence of actions: 
\begin{enumerate}
    \item \textbf{Probing}: All data centers acquire the latest global model. Each data center then selects a subset of local data, constituting $\epsilon$ percent of the total, to calculate an update gradient before training. The time cost of calculating such a gradient can be omitted when $\epsilon$ is small. Then the data centers send their gradients to the server.
    \item \textbf{Client Selection}: Based on the local gradients received from all data centers, $K^t$ data centers (where we denote the set of selected data centers as $\mathcal{K}^t$)  are selected to participate in this training time slot. These selected data centers perform the $M$ local training epochs, while the data centers not chosen take a break.
    \item \textbf{Local training}: The chosen $K^t$ data centers engage in communication with the server during each of the total $M$ training epochs. In each of these epochs, the server aggregates the local updates into a new model and then broadcasts this model to the selected $K^t$ data centers for the next round of local epoch training. 
\end{enumerate}

\textbf{Remark:} We observe that the primary distinctions between our proposed training process and classic FL are the inclusion of a probing step and the ability for data centers to communicate with the server during the local training. While this may introduce additional communications, it is justifiable in our considered setting. Traditional FL typically involves training among mobile devices, which often face communication constraints. In contrast, our approach focuses on training among data centers, which generally possess robust communication capabilities, making the need for extra communication less significant.

Upon understanding the training process of our proposed method, CAFE, the subsequent step involves determining the client selection for each time slot. Given the non-IID nature of datasets at individual data centers, it is crucial to strategically select data centers for each training slot, balancing carbon footprint efficiency with learning performance. These considerations will be elaborated and formulated in subsequent sections. Let \(a^t_i \in \{1, 0\}\) indicate the selection status of data center \(i\) during time slot \(t\), and \(\boldsymbol{a}^t = (a^t_1, ..., a^t_N)\) denote the aggregate data center selection decisions. Consequently, the set of selected data centers for time slot \(t\), \(\mathcal{K}^t\), is given by \(\mathcal{K}^t = \{i \mid a_i^t = 1\}\).

\subsection{Carbon Footprint}
Training large-scale AI models necessitates substantial energy consumption, primarily to power GPU operations. This high energy demand, particularly for electricity, results in significant carbon emissions, as electricity generation often relies on carbon-intensive fuels like coal. Therefore, we initially define the energy consumption for each data center during a specific training time slot $t$. It is clear that if a data center $i$ is chosen to participate in the training during time slot $t$ (i.e., $a_i^t = 1$), it will consume more energy. We represent the energy consumption of data center $i$ during the training time slot $t$ as follows:
\begin{align}
    e_i^t(a_i^t) = E_{i}^{c} + a_i^t E_{i}^{s}, 
\end{align}

Here, $E_{i}^c$ denotes the constant energy consumption during time slot $t$, which includes the static energy consumption when data center $i$ is not processing any workload. $E_i^s$ signifies the energy consumption when data center $i$ is selected to participate, which includes the energy used for the $M$ local epoch training sessions.

Next, we define the carbon intensity of data center $i$ during the training time slot $t$ as $\beta_i^t$, measured in kg per kWh. This carbon intensity $\beta_i^t$ can be determined by consulting the local utility or by averaging the carbon intensity of the grid's fuel mix. Studies, such as \cite{li2023towards}, have confirmed that carbon intensity varies based on the data center's location and time. Given the carbon intensity, we can calculate the carbon footprint for data center $i$ at time slot $t$ as:
\begin{align}
    c_i^t(a_i^t | \beta_i^t) =  \beta_i^t e_i^t(a_i^t).
\end{align}

The total carbon footprint at time slot $t$ can be calculated as:
\begin{align}
    c^t(\boldsymbol{a}^t| \boldsymbol{\beta}^t) = \sum_{i \in \mathcal{N}}  \beta_i^t e_i^t(a_i^t),
\end{align}
where $\boldsymbol{\beta}^t = (\beta_1^t, \dots, \beta_N^t)$.

\subsection{Learning Performance}
The training goal is to collaboratively train a machine learning model using the local dataset $\mathcal{D}_i$ stored at each data center $i$. This essentially involves solving a distributed optimization problem as follows:
\begin{align}
    \min_w f(w) = \frac{1}{N}\sum_{i=1}^N f_i(w).
\end{align}
Here, $f_i:\mathbb{R}^d \to \mathbb{R}$ is a loss function (typically non-convex in deep learning training scenarios) for data center $i$, and $w$ represents the model parameters we aim to train. This distributed optimization problem can be addressed by applying gradient descent-related techniques. During each training time slot $t$, only the selected data centers (i.e., $a_i^t = 1$) participate in the $M$ local training epochs. Therefore, for most of the time during each training time slot, we can only access the gradients of selected data centers. However, to obtain the most accurate and stable direction towards the minimum of the loss function, the ideal aggregated gradient is $\frac{1}{N}\sum_{i \in \mathcal{N}} \nabla f_i (w)$, which assumes all data centers participate. Although involving all data centers could improve training performance, it is not feasible to have all data centers participate in all training time slots due to the essential concern of limiting the total carbon footprint. 

Consequently, during each time slot $t$, we need to assess how well the gradients from the selected set $\mathcal{K}^t$ approximate the ideal aggregated gradient and use this assessment to reflect the learning performance. Inspired by \cite{mirzasoleiman2020coresets}, we decide to employ the concept of coreset. The fundamental idea of coreset is to represent elements not in the coreset using elements that are in the coreset. Therefore, we introduce the following metric to measure how well the selected set $\mathcal{K}^t$ serves as a `coreset', and further use it to describe the learning performance in time slot $t$:
\begin{align}
\label{metric}
    U^t(\boldsymbol{a}^t) = b - \sum_{j \in \mathcal{N}} \min_{i \in \mathcal{K}^t} \|\nabla f_j(w) - \nabla f_i(w)\|,
\end{align}
where $b$ is a constant positive number that is large enough (i.e. $ U^t(\boldsymbol{a}^t) \geq 0$ always holds) and the selected set can be expressed by $\mathcal{K}^t = \{i \mid a_i^t = 1\}$. When no data center is selected, the utility value is assigned a value of $U^t =  b - 2 N \max_{i \in \mathcal{N}} \| \nabla f_i(w)\|$. It is also worth noting that with different model parameters $w_1$ and $w_2$, the best approximating data center for data center $i$ can change.

\textbf{Remark:} To compute the utility metric $U^t$ for each time slot, it is necessary to calculate the gradient $\nabla f_i(w)$ for all data centers $i$. However, obtaining such gradients requires each data center to utilize all its data samples, which is not environmentally efficient. To address this, we introduce the probing step where we randomly select $\epsilon$ percent of the data samples to estimate $\nabla f_i(w)$. While there may be some discrepancies between the precise and estimated values, for the sake of theoretical simplicity, we omit these differences in the subsequent theoretical analysis. Furthermore, based on our simulation experiments, the probing step provides reasonably accurate estimates, and such differences can be disregarded.

\subsection{Problem Formulation}
Our focus is on maximizing the learning performance during the $T$ time slots. The goal is to optimize the average value of the metric, as defined in Eq. \ref{metric}. At the same time, we need to adhere to a long-term carbon constraint while determining the most suitable data center selection. Suppose that the total carbon constraint for the entire training procedure is represented by $H$. We can express this problem in the following way:
\begin{align}
\textbf{P1}&~~~\max_{\boldsymbol{a}^0, ..., \boldsymbol{a}^{T-1}} \frac{1}{T}\sum_{t=0}^{T-1} U^t(\boldsymbol{a}^t)\\
\text{s.t.}&~~~\sum_{t=0}^{T-1} c^t(\boldsymbol{a}^t| \boldsymbol{\beta}^t) \leq H \label{con:carbon}\\
&~~~a^t_i \in \{0, 1\}, \forall i, \forall t \label{con:selection}
\end{align}

Constraint \eqref{con:carbon} ensures that the total carbon emissions throughout the $T$ slots from all data centers must not surpass the budget $H$. Constraint \eqref{con:selection} specifies data center selection be binary. 

Solving \textbf{P1} is challenging due to our lack of future information. To optimally resolve \textbf{P1}, we would need complete offline data about the entire training process (i.e. carbon intensity), which is notably difficult to predict accurately in advance. Therefore, an online approach that can efficiently determine data center selection without the need to predict future events is necessary to be developed.

\subsection{Offline Benchmark}
Before we introduce the online algorithm, we first present a virtual case as a benchmark. Unlike classical online decision-making problems, the unique challenge in our situation is that the objective function is contingent on the model $ w $ provided at each time slot, which changes based on previous decisions. We consider a virtual model sequence $\{w^{full,1}, \dots, w^{full,T}\}$, where each $ w^{full,t} $ symbolizes the model fully averaged by \textbf{all} data centers to participate in all local $ M $ epochs at every time slot $ t $. This virtual model sequence is constructed to be independent of previous selection decisions, thus allowing us to examine a more straightforward scenario. With these virtual model sequences, and assuming that all $ T $ time slots' carbon intensity information (i.e., $\boldsymbol{\beta}^t$ for each $ t $) is known, we can solve the problem \textbf{P1} offline. We denote the optimal learning performance achieved in this virtual case as $\text{OPT}$, and the selection decision for each time slot in this situation is denoted as $\boldsymbol{a}^{opt,t}$.

It is important to note that the fully-participated model sequence $\{w^{full,1}, \dots, w^{full,T}\}$ is impractical to be known. Furthermore, predicting future carbon intensity information is also infeasible, as the carbon intensity varies during time slots and can be influenced by multiple factors. Consequently, the virtual selection decisions $\boldsymbol{a}^{opt,t}$ for each time slot $ t $ cannot be practically achieved. The purpose of introducing this virtual case is to utilize it as a benchmark for our practical online algorithm, which we will propose in the following section. By setting this theoretical standard, we can better assess the performance and potential limitations of our solution.

\section{Online Data Center Selection}

\subsection{Methodology}
Our online method is built upon the Lyapunov drift-plus-penalty framework, taking into account finite $T$ time slots. We create a virtual carbon deficit queue, denoted by $q^t$, which helps guide the selection of data centers in accordance with the total carbon footprint constraint. Notably, while the original Lyapunov framework focuses on infinite $T$, our context requires a finite number of total training slots $T$. To address this, we follow \cite{wang2023federated}'s approach, setting the initial value of the virtual queue as $q^0 > 0$. The queue then gets updated at the conclusion of each time slot $t$, as follows:
\begin{align}\label{queue}
q^{t+1} = max\{0, c^t - H/T + q^t\},
\end{align}
where $c^t$ represents $c^t(\boldsymbol{a}^t| \boldsymbol{\beta}^t)$. Essentially, the virtual queue captures the accumulated transgressions of the constraint. Therefore, our goal is to both maximize the metric as defined in Eq. \ref{metric}, and minimize the size of the virtual queue. We introduce a constant $V > 0$, which we will discuss in greater detail later. For each time slot $t$, we aim to resolve the following per-slot problem:
\begin{align}
\textbf{P2}&~~~\max_{\boldsymbol{a}^t}  V\cdot U^t(\boldsymbol{a}^t) - q^t \cdot c^t\\
\text{s.t.}&~~~ \text{Selection constraints:} ~~~~\eqref{con:selection}
\end{align}

By introducing the additional term $q^t \cdot c^t$, the system now factors in the carbon deficit of the data centers during the current time slot's data center selection. As such, we need to strike a balance between maximizing learning performance (as defined by the metric) and minimizing carbon emissions. The performance of the online algorithm will be assessed in the following section. For now, we will concentrate on how to resolve the per-slot problem.

\subsection{Per-Slot Problem}
While \textbf{P2} does not depend on future information, it remains difficult to solve due to the binary selection constraints for each $a_i^t$. Basic methods, such as Exhaustive Search, can require prohibitively large computations. For instance, an Exhaustive Search approach requires a total of $2^N - 1$ computations. This method could work when the number of data centers is small (i.e., $N \leq 20$). 

However, with technological advancements, data collection capabilities have increased, necessitating larger models for training and leading to the establishment of more data centers. Consequently, methods like Exhaustive Search become impractical. To tackle this challenge, we begin by showing that the objective function of \textbf{P2} is a submodular function. The first component of the objective function is a widely recognized submodular function, specifically, the facility location function \cite{CORNUEJOLS1977163}. The second component of the objective function is a linear function, which does not influence the overall submodularity of the function. Therefore, we can conclude that the objective function of \textbf{P2} is a submodular function.

In the objective function of \textbf{P2}, an additional term is incorporated to represent a negative carbon cost. This introduces a challenge in ensuring that the inclusion of an extra data center consistently yields a positive impact on the objective function. Furthermore, within this context, there exists no limitations on the number of data centers that can be chosen during each time slot. Consequently, the problem at hand is categorized as an unconstrained non-monotone submodular maximization problem. It is worth highlighting that our analysis assumes the objective function value in \textbf{P2} remains non-negative. This type of problem is well-recognized within the field and is acknowledged as NP-hard.

The aforementioned unconstrained non-monotone submodular maximization problem has undergone extensive investigation in recent years. Numerous algorithms \cite{zhang2023online, feige2011maximizing, lee2009non, buchbinder2015tight} have been developed to address this problem, often achieving a $1/\gamma$-approximation, as defined in Definition \ref{gamma_optimal} for optimal assurance.

\begin{definition}
\label{gamma_optimal}
    Let $\Bar{\boldsymbol{a}}^{t} $ represent the optimal solution for \textbf{P2} during time slot $t $. The solution generated by existing non-monotone submodular algorithms (denoted as $\boldsymbol{a}^{*,t} $) provides a $\frac{1}{\gamma} $-approximation if:
    $g(\boldsymbol{a}^{*,t}) \geq \frac{1}{\gamma} g(\Bar{\boldsymbol{a}}^{t} )$, where $g(.)$ denotes the objective function of \textbf{P2} and $\gamma \geq 1$.
\end{definition}

Typically, the Exhaustive Search approach can indeed achieve the optimal solution for \textbf{P2}, resulting in $\gamma = 1$ for this method. The degree of approximation is contingent upon the algorithm employed for resolution. Here, we introduce an algorithm proposed by \cite{buchbinder2015tight} allowing us to tackle the aforementioned submodular problem with a $1/3$-approximation.

% We proceed by \textbf{randomly} shuffling the entire set of data centers, $\mathcal{N} $, to the shuffled set denoted as $\Tilde{\mathcal{N}} $. The randomized double greedy algorithm then sequentially operates on this shuffled set.

\textbf{Deterministic Double Greedy Algorithm}: This method is proposed by \cite{buchbinder2015tight}. We modify it to align with our setting.  For each specific time slot $t $, two vectors, $\boldsymbol{a}^e $ and $\boldsymbol{a}^f $, are defined. Initially, $\boldsymbol{a}^e = (a^t_1 = 0, \dots, a^t_N = 0) $, meaning \textbf{no} data center is selected, and $\boldsymbol{a}^f = (a^t_1 = 1, \dots, a^t_N = 1) $, implying \textbf{all} data centers are selected. The deterministic double greedy algorithm sequentially operates on the data centers. Within each data center $j \in \mathcal{N} $, we compute two specific values. The first, denoted as $u_j $, evaluates the difference in objective value when virtually setting $a^t_j = 1 $ in $\boldsymbol{a}^e $:
\begin{align}
    u_j = g(\boldsymbol{a}^e_{-j}, a_j^e = 1) - g(\boldsymbol{a}^e) 
\end{align}
Here, $g(.) $ symbolizes the objective function of \textbf{P2}, and $\boldsymbol{a}^e_{-j} $ represents the selection of all data centers excluding $j $ in $\boldsymbol{a}^e $.

The second value, $v_j $, measures the difference in objective value when setting $a^t_j = 0 $ in $\boldsymbol{a}^f $:
\begin{align}
    v_j = g(\boldsymbol{a}^f_{-j}, a_j^f = 0) - g(\boldsymbol{a}^f) 
\end{align}

Next, we proceed to compare the values of $v_j$ and $u_j$. In the case where $u_j$ is greater than or equal to $v_j$, the subsequent operations are carried out as follows:
\begin{align}
\label{ugeater}
    &\boldsymbol{a}^e \gets (\boldsymbol{a}^e_{-j}, a_j^e = 1);  \quad
    \boldsymbol{a}^f \gets \boldsymbol{a}^f
\end{align}
Conversely, if $u_j$ is less than $v_j$, the following adjustments are made:
\begin{align}
\label{vgeater}
    &\boldsymbol{a}^e \gets \boldsymbol{a}^e;  \quad
    \boldsymbol{a}^f \gets (\boldsymbol{a}^f_{-j}, a_j^f = 0)
\end{align}

% We further define $v_j^+ = \max\{v_j, 0\} $ and $u_j^+ = \max\{u_j, 0\} $. Based on these two non-negative numbers, with a probability $p = \frac{u_j^+}{u_j^+ + v_j^+} $, the following operations are executed:
% \begin{align}
%     &\boldsymbol{a}^e = (\boldsymbol{a}^e_{-j}, a_j^e = 1)\nonumber\\
%     &\boldsymbol{a}^f = \boldsymbol{a}^f
% \end{align}
% With the complementary probability, $1 - p $, we undertake:
% \begin{align}
%     &\boldsymbol{a}^e = \boldsymbol{a}^e \nonumber\\
%     &\boldsymbol{a}^f = (\boldsymbol{a}^f_{-j}, a_j^f = 0)
% \end{align}

Upon completion of all $N $ iterations over elements in the data centers set, $\boldsymbol{a}^{*,t} $, is obtained through $\boldsymbol{a}^{*,t} = \boldsymbol{a}^e = \boldsymbol{a}^f $.

\textbf{Randomized Double Greedy Algorithm:} The aforementioned deterministic algorithm ensures a minimum of a $1/3$-approximation to the optimal solution. However, by introducing randomization into the algorithm, significant improvements can be achieved. The enhanced randomized algorithm can attain a considerably tighter at least $ 1/2$ approximation, considering the expected outcomes resulting from the algorithm's random choices.

The fundamental process of the randomized algorithm resembles that of the deterministic version. Nonetheless, certain modifications are incorporated to incorporate randomization. After calculating the values $u_j$ and $v_j$, two new non-negative values are defined: $v_j^+ = \max\{v_j, 0\}$ and $u_j^+ = \max\{u_j, 0\}$. Leveraging these non-negative values, the algorithm operates as follows: with a probability of $p = \frac{u_j^+}{u_j^+ + v_j^+}$, the operations outlined in Eq. \eqref{ugeater} are executed. Conversely, with a complementary probability of $1 - p$, the operations detailed in Eq. \eqref{vgeater} are undertaken.

\vspace{-10pt}
\begin{algorithm}
\caption{CAFE}
\begin{algorithmic}[1]
\For{$t = 1$ to $T$}
    \State {broadcast} global model $w^t$ to all data centers $i \in \mathcal{N}$
    \State \textbf{1. Probing step (All data centers):}
    \For{each data center $i \in \mathcal{N}$}
        \State Estimate $\nabla f_i(w^t)$ based on $\epsilon \%$ of its total data samples.
        \State Communicate $\nabla f_i(w^t)$ to the `server' data center.
    \EndFor
    \State \textbf{2. Selecting data center step (Server):}
    \State Solve problem \textbf{P2} based on the $\nabla f_i(w^t)$ and \par each data center's current carbon intensity $\beta^t_i$.
    \State Use the solution of \textbf{P2} to determine the selected \par data center set $\mathcal{K}^t$.
    \State \textbf{3. Local training (Selected data centers) and aggregating step (Server):}
    \For{local epoch $m = 1$ to $M$}
     \For{each data center $i \in \mathcal{K}^t$}
        \State Perform the local training.
    \EndFor
    \State Aggregate the local updates and broadcast to the selected data centers.
    \EndFor
\EndFor
\State Update the global model $w^{t+1}$.
\end{algorithmic}
\label{alg}
\end{algorithm}
\vspace{-10 pt}
By solving problem \textbf{P2} using either method, we can now determine the data center selection for each time slot in our proposed method, CAFE. We summarize CAFE in Algorithm \ref{alg}.

% \begin{proposition}
%     Assuming $\Bar{\boldsymbol{a}}^{t} $ is the optimal solution for \textbf{P2} at time slot $t $, the solution acquired by the randomized double greedy algorithm (i.e., $\boldsymbol{a}^{*,t} $) offers a $\frac{1}{2} $-approximation:
%     \begin{align}
%         \mathbb{E}(g(\boldsymbol{a}^{*,t})) \geq \frac{1}{2} g(\Bar{\boldsymbol{a}}^{t} ),
%     \end{align}
%     where $g(.)$ is the objective function of \textbf{P2} and expectations are taken over the random choices of the algorithm.
% \end{proposition}

\subsection{Performance Analysis}
Now, we discuss the optimality and constraint satisfaction achieved when we solve \textbf{P1} approximately by minimizing the drift-plus-penalty objective of \textbf{P2}. To streamline our analysis, we introduce the following assumptions.

\begin{assumption}
\label{ass1}
For each training time slot $t$, the total footprint created by constant energy consumption is smaller than $H/T$.
\begin{align}
     \sum_{i \in \mathcal{N}} \beta_i^t E_i^c - H/T \leq 0, ~~~ \forall t
\end{align}
\end{assumption}

\begin{assumption}
\label{ass2}
There exits a constant $G > 0$ such that $\|\nabla f_j (w)\| < G, ~~\forall j \in \mathcal{N}, \forall w$.
\end{assumption}

% \begin{assumption}
%     There exits a constant $\delta > 0$, such that with the same selection decision $\boldsymbol{a^t}$, the absolute difference between the utility value based different model parameters $w_1$ and $w_2$ is bounded.
% \label{ass3}
% \begin{align}
%     |U^t(\boldsymbol{a^t}|w_1) - U^t(\boldsymbol{a^t}|w_2)| < \delta
% \end{align}
% \end{assumption}

\begin{assumption}
    There exits a constant $\delta > 0$, such that the gradient divergence is bounded as:
    \begin{align}
        \|\nabla f_i (w) - \nabla f(w)\| \leq \delta, \forall w, i
    \end{align}
\end{assumption}

% \begin{assumption}
% \label{ass3}
% There exist a constant $\delta > 0$, such that for any two data centers $i,j \in \mathcal{N}$, with different model parameters $w_1$ and $w_2$, it has the following bound:
% \begin{align}
%    | \|\nabla f_j(w_1) - \nabla f_i(w_1)\| - \|\nabla f_j(w_2) - \nabla f_i(w_2)\| | \leq \delta
% \end{align}
% \end{assumption}

Assumption 1 ensures that the total carbon footprint budget, represented by $ H $, is large enough to cover the carbon emissions produced by the static energy of all data centers when they are not processing any workload. Assumptions 2 and 3 are frequently employed in the convergence analysis of distributed learning, as evidenced by references such as \cite{li2019convergence, yu2019parallel, yang2021achieving}. Specifically, Assumption 2 asserts that the gradient remains uniformly bounded, a condition crucial for ensuring stability and consistent control throughout the optimization phase. On the other hand, Assumption 3 pertains to the extent of non-IID data distribution across various data centers.
% \begin{assumption}
% There exits a constant $\delta > 0$, such that with the same selection decision $\boldsymbol{a^t}$, the difference between the expected utility value caused by different model parameters $w$ is bounded.
% \label{ass3}
% \begin{align}
%     \mathbb{E}[U^t(\boldsymbol{a^t})| w_1] - \mathbb{E}[U^t(\boldsymbol{a^t})| w_2] < \delta
% \end{align}
% \end{assumption}

\begin{proposition}
\label{prop3}
    With the same selection decision $\boldsymbol{a^t}$, the absolute difference between the utility value based on different model parameters $w_1$ and $w_2$ is bounded as follows:
    \begin{align}
        |U^t(\boldsymbol{a^t}|w_1) - U^t(\boldsymbol{a^t}|w_2)| \leq 4 N \delta
    \end{align}
\end{proposition}

% \begin{proof}
%     We start with the definition of utility defined in Eq. \ref{metric}.
%     \begin{align}
%          & |U^t(\boldsymbol{a^t}|w_1) - U^t(\boldsymbol{a^t}|w_2)| \nonumber\\
%          = &  |\sum_{j \in \mathcal{N}} \min_{i \in \mathcal{K}^t} \|\nabla f_j(w_2) - \nabla f_i(w_2)\| -  \|\nabla f_j(w_1) - \nabla f_i(w_1)\|| \nonumber\\
%          \leq &  \sum_{j \in \mathcal{N}} \min_{i \in \mathcal{K}^t} \|\nabla f_j(w_2) - \nabla f_i(w_2)\|  +  \|\nabla f_j(w_1) - \nabla f_i(w_1)\| \nonumber\\
%          = &  \sum_{j \in \mathcal{N}} \min_{i \in \mathcal{K}^t}   \|\nabla f_j(w_2) - \nabla f(w_2) + \nabla f(w_2) - \nabla f_i(w_2)\| \nonumber\\
%           +&   \sum_{j \in \mathcal{N}} \min_{i \in \mathcal{K}^t}   \|\nabla f_j(w_1) - \nabla f(w_1) + \nabla f(w_1) - \nabla f_i(w_1)\|  \nonumber\\
%           \leq &  \sum_{j \in \mathcal{N}} \min_{i \in \mathcal{K}^t}  ( \|\nabla f_j(w_2) - \nabla f(w_2)\| + \|\nabla f(w_2) - \nabla f_i(w_2)\|) \nonumber\\
%           +&   \sum_{j \in \mathcal{N}} \min_{i \in \mathcal{K}^t}   (\|\nabla f_j(w_1) - \nabla f(w_1)\| + \|\nabla f(w_1) - \nabla f_i(w_1)\| )\nonumber\\
%           \leq &  \sum_{j \in \mathcal{N}} 2\delta + \sum_{j \in \mathcal{N}} 2\delta = 4 N \delta
%     \end{align}
%     where the first and second inequality is based on triangle inequality and the last inequality is based on Assumption 3.
% \end{proof}
\begin{proof}
The proof is shown in Appendix [\ref{proof_p1}].
\end{proof}

\begin{theorem}
Under Assumption \ref{ass1}, \ref{ass2}, solving the per-slot problem \textbf{P2} with $1/\gamma$-approximation in each time slot $t$ ensures the following bound on the constraint violation:
\begin{align}
    &\frac{1}{T}\sum_{t=0}^{T-1} c^t - \frac{H}{T} 
    \leq   \sqrt{\frac{(q^0)^2}{T^2} + \frac{\frac{2V}{\gamma}(b + N G) + 2B_1}{T}} - \frac{q^0}{T},
\end{align}
where $B_1 > 0$ is a finite constant number.
\label{thm1}
\end{theorem}

\begin{proof}
The proof is shown in Appendix [\ref{proof_t1}].
\end{proof}

\textbf{Remarks:} Theorem 1 provides insights into the behavior of the constraint in relation to system parameters, such as \(q^0\) and \(V\). As \(T\) approaches infinity, the budget constraint is satisfied. For a finite \(T\), the right-hand side (RHS) of the constraint decreases as \(q^0\), the initial queue length, increases. Specifically, as \(q^0\) approaches infinity, the RHS tends toward 0. This behavior is intuitive: a larger initial queue length causes the algorithm to impose a significant penalty for any violation of the budget constraint. Additionally, Theorem 1 highlights the impact of the parameter \(V\) on the budget's restriction. A larger \(V\) increases the likelihood of violations, consistent with the idea that a higher \(V\) indicates the algorithm's emphasis on maximizing the utility value \(U^t\). Consequently, the algorithm is more likely to disregard violations of the budget constraint.

\begin{theorem}
    Under Assumption \ref{ass1}, \ref{ass2} and Proposition \ref{prop3} solving the per-slot problem \textbf{P2} with $1/\gamma$-approximation in each time slot $t$ ensures the following bound related to the objective of \textbf{P1}
\begin{align}
& \frac{1}{T}\sum_{t=0}^{T-1}U^t(\boldsymbol{a^t}) \geq   \frac{\text{OPT}}{\gamma}-\frac{4}{\gamma}N\delta- \frac{1}{V}(\frac{q^0c^{max}}{\gamma} + \frac{1}{2T}(q^0)^2) \nonumber\\
&- \frac{1}{V}(\frac{(T c^{max} - H)c^{max}}{\gamma} + B_1)
\end{align}
where $\text{OPT}$ denotes the offline optimal value with fully-participated model sequence and all-time slots' carbon intensity information. However, we note that the above sequence and information are not achievable in practice. 
\label{thm2}
\end{theorem}

\begin{proof}
The proof is shown in Appendix [\ref{proof_t2}].
\end{proof}

\textbf{Remark}: 
Theorem 2 elucidates the performance attributes of our algorithm. For the per-slot problem, denoted as \textbf{P2}, the solver guarantees a minimum of $\frac{1}{\gamma}$-approximation. Consequently, the long-term average expected value also aligns with this approximation. The optimality gap includes components $G_1$, $G_2$, and $G_3$, and is shaped by several factors. Given our emphasis on finite training time slots, mirroring real-world conditions, our analysis centers on $\delta$, $q^0$, and $V$, treating $T$ as a fixed constant. Specifically, $G_1$ is primarily influenced by $\delta$, denoting the Non-IID degree among data centers. A higher Non-IID degree expands the optimality gap, while a more IID data distribution narrows it. Components of $G_2$ are shaped by both $V$ and the initial queue length $q^0$, while $G_3$ is contingent on $V$. A larger $q^0$ accentuates the optimality gap, prompting a more conservative algorithmic approach to respect budget limits, potentially at the expense of performance. However, as $V$ grows, both $G_2$ and $G_3$ diminish, indicating the algorithm's shift towards prioritizing performance over strict carbon budget adherence.

\section{Simulation Results}
In this section, we evaluate our proposed online algorithm in comparison with several baseline methods and conduct ablation studies to assess the influence of different parameters.

\subsection{Setup}
\subsubsection{Data Centers}
We consider a collection of 30 geographically distributed data centers. We utilize the carbon intensity data from Electricity Maps. These data centers are extensive in scale, akin to significant establishments like Google's data centers. Each data center in our consideration possesses 2000 NVIDIA A100 GPUs dedicated to the training of large-scale models. The power utilization for each A100 GPU stands at 400 watts under full load conditions, while idling operation consumes approximately 20 watts, corresponding to periods when no active training tasks are being executed.

Given the contemporary trend of increasing size and intricacy in large models, we take into account a comprehensive training period of 10 days for mastering such a sizable model. Our temporal framework divides this span into discrete time slots, with each slot denoted as $t$ and lasting for 1 hour.

\subsubsection{Simulated Learning Tasks}
Our objective is to explore ways to allocate large model training tasks across data centers while adhering to a total carbon emission constraint. As delineated in the previous section, it is essential to monitor the learning performance for each time slot $ t $ as represented by Eq. \ref{metric}. Given the substantial resource requirements for training large models and our limited computational resources, we opt to use a smaller model task to emulate the training process of a more extensive model. Specifically, we are considering the classic classification training tasks on the CIFAR-10 and CIFAR-100 \cite{krizhevsky2009learning} datasets, employing ResNet \cite{he2016deep} as the backbone model. Our simulation experiments are based on one Geforce RTX 3080 GPU. To accurately simulate the scenario where data centers use A100 GPUs for large model training, we scale up the time required to train smaller models using 3080 GPU in our simulated learning environment. This scaled time serves as an approximation of the actual training time needed for large model training in data centers. While our proposed online learning method can be applied to large model training involving collaborative efforts across multiple data centers, our choice to use the smaller model here is merely to simulate learning performance due to computational constraints. 

\subsubsection{Default parameters}
In the default configuration, we assume that the data stored by each data center is non-IID. To simulate non-IID, we utilize the Dirichlet Distribution. This heterogeneity is quantified using a Dirichlet parameter set to 0.8. Each time slot is designated as a 1-hour interval, during which data centers execute $M = 2$ training epochs. We initially set the percentage of the small representative dataset $\epsilon = 0.05$ for training during the beginning of each time slot. The training time cost for this process is omitted for simplicity. Our experiment spans a total of $T = 200$ time slots and operates under a carbon footprint budget of $H = 400$ tons. For the Lyapunov control parameters, we set the Lyapunov parameter $V = 0.5$ and the initial queue value $q^0 = 10$. To ensure the reliability of our results, each simulation is conducted three times at random, and the average of these runs is taken as the final outcome.
\begin{figure}[h]
\centering
\vspace{-5 pt}
\subfigure[Convergence]{
  \includegraphics[width=0.475\linewidth, height=0.313\linewidth]{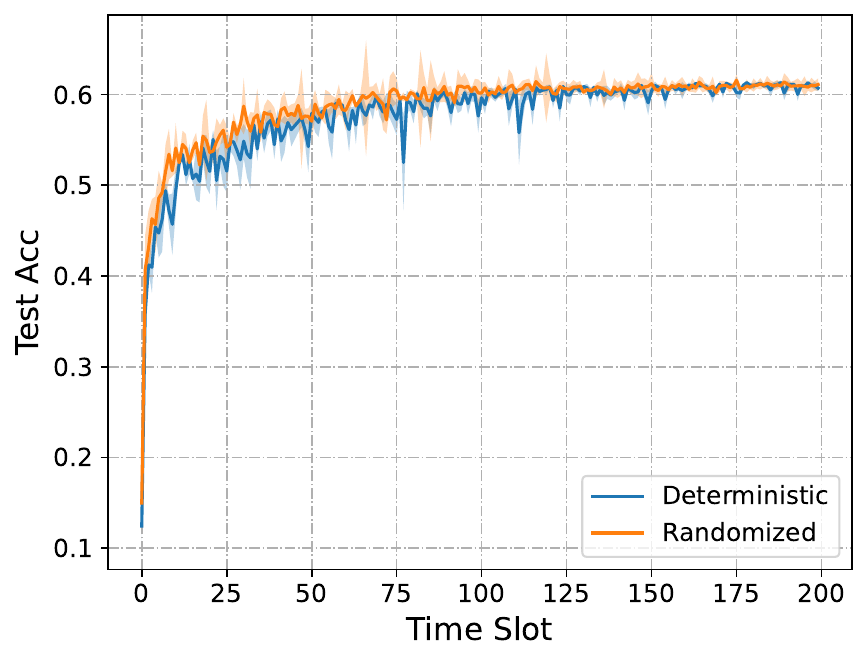}
  \label{fig:1sub1}}
\subfigure[Utility and Carbon Footprint]{
  \includegraphics[width=0.475\linewidth, height=0.313\linewidth]{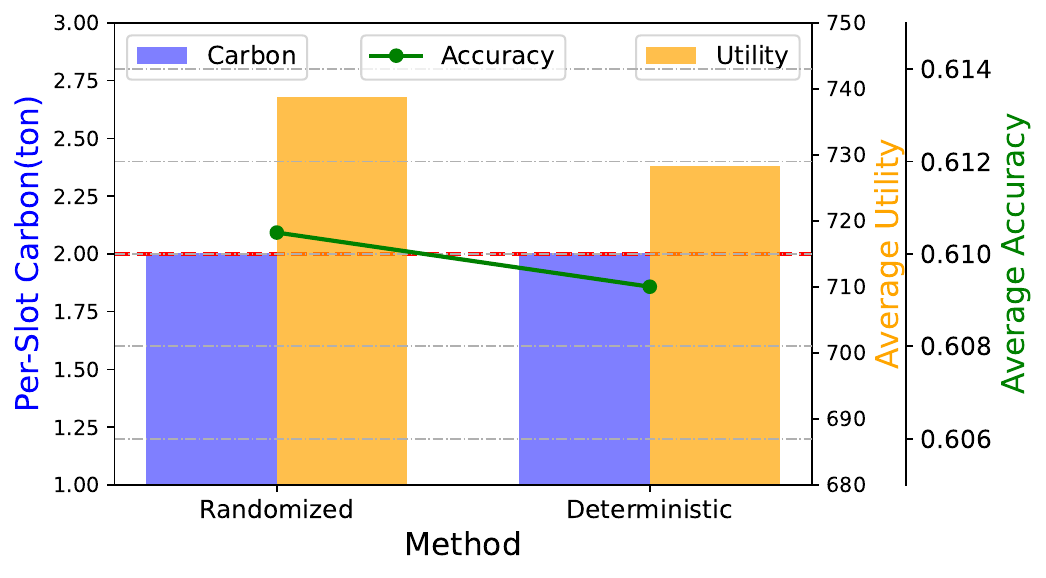}
  \label{fig:1sub2}}
  \vspace{-5 pt}
\caption{Impact of per-slot algorithms.}
\label{fig:per-slot}
\vspace{-20 pt}
\end{figure}

\begin{figure}[h]
\centering
\subfigure[Convergence]{
  \includegraphics[width=0.475\linewidth, height=0.313\linewidth]{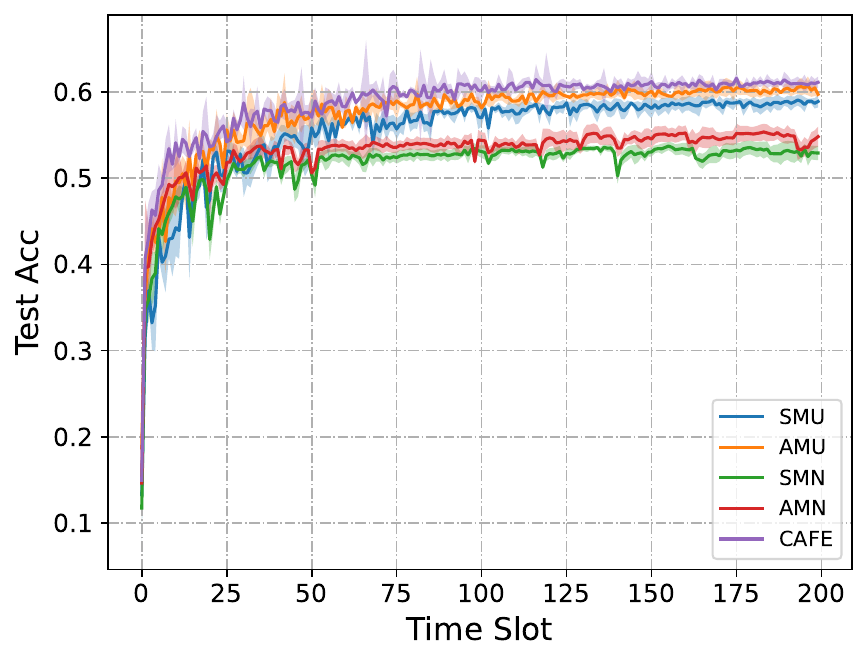}
  \label{fig:2sub1}}
\subfigure[Accuracy and Carbon Footprint]{
  \includegraphics[width=0.475\linewidth, height=0.313\linewidth]{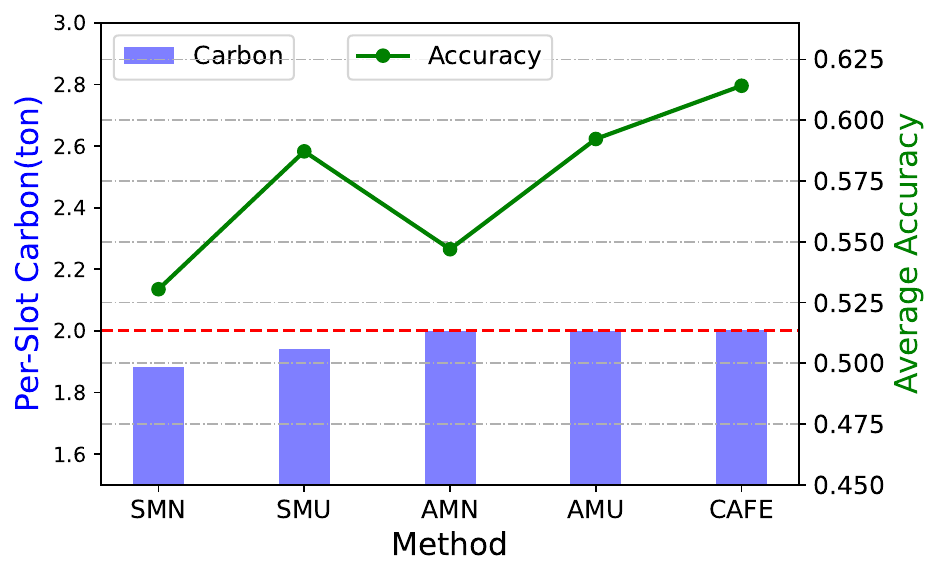}
  \label{fig:2sub2}}
  \vspace{-5 pt}
\caption{Performance Comparison on CIFAR-10.}
\label{fig:baseline}
\vspace{-15 pt}
\end{figure}

\subsubsection{Benchmarks}
We compare the performance of the proposed online algorithm with the following benchmark algorithms. 

% \textbf{Full Participate (FP):}
% In this approach, we assume that all 10 data centers participate in the training during each time slot. This method undoubtedly exceeds the carbon footprint budget $ H $. We use this approach as an upper-bound baseline for comparison.

\textbf{Static Myopic Utility (SMU):}
In this approach, we employ a uniform allocation strategy, dividing the total carbon budget equally across each time slot. The available carbon budget for each slot is denoted by $ H/T $. Following this, we maximize $U^t$ slot-by-slot, ensuring compliance with the budget constraint for each slot.

\textbf{Static Myopic Number (SMN):}
This approach is similar to the above SMU, with the key difference being that the per-slot objective focuses on maximizing the number of participating data centers, $ K^t $, rather than optimizing the utility $ U^t $.

\textbf{Adaptive Myopic Utility (AMU):} 
A limitation of the Static Myopic approach is the risk of budget inefficiency, as the assigned budget for each time slot might not be fully expended. The Adaptive Myopic strategy addresses this by redistributing any unspent budget uniformly across the subsequent time slots. In this method, the budget available during time slot $ t $ is calculated as $ (H - \tilde{H})/(T-t) $, where $ \tilde{H} $ indicates the cumulative carbon emission used up to that moment. The objective focuses on maximizing $U^t$.

\textbf{Adaptive Myopic Number (AMN):}
This approach is similar to AMU, with changing the objective to maximize the number of participating data centers $K^t$.

% When comparing our proposed online algorithm with benchmark algorithms, we employ Exhaustive Search across all methods to ensure a fair comparison. In subsequent sections, we will also evaluate the performance of our proposed algorithms using both Exhaustive Search and the Randomized Double Greedy method. The latter offers significant computational savings and is more practical when dealing with a large number of data centers.

\begin{figure}[h]
\centering
\subfigure[Convergence]{
  \includegraphics[width=0.475\linewidth, height=0.313\linewidth]{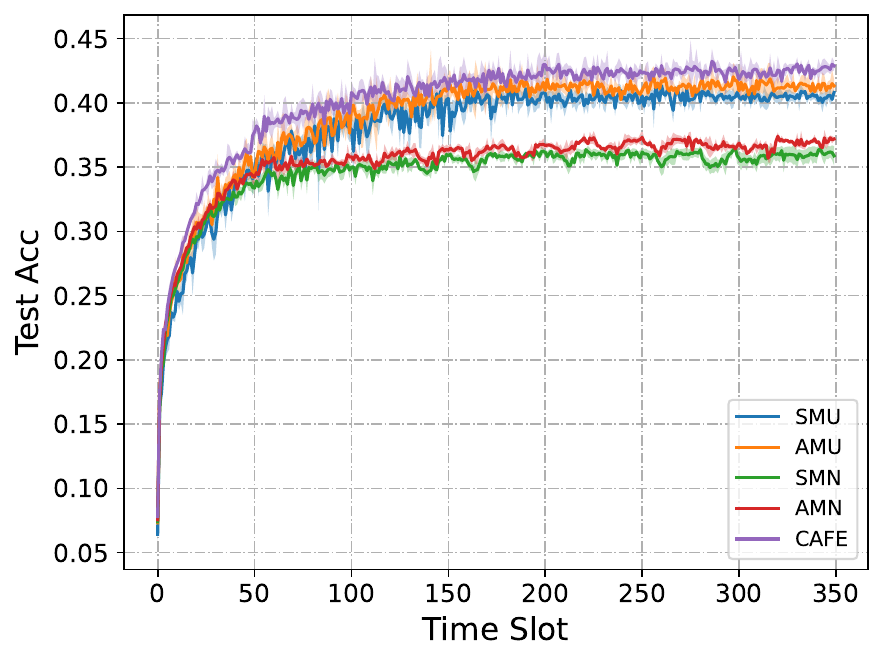}
  \label{fig:ciar100-sub1}}
\subfigure[Accuracy and Carbon Footprint]{
  \includegraphics[width=0.475\linewidth, height=0.313\linewidth]{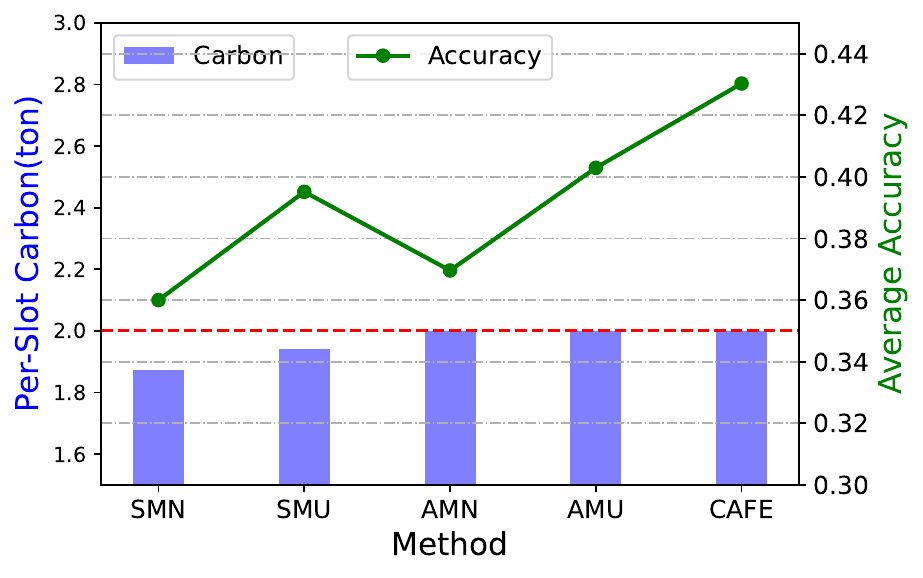}
  \label{fig:ciar100-sub2}}
  \vspace{-5 pt}
\caption{Performance Comparison on CIFAR-100.}
\label{fig:baseline_cifar100}
\vspace{-15 pt}
\end{figure}

\subsection{Simulation Results}
\subsubsection{Impact of Per-slot Algorithms}
Before comparing our approach to baseline methods, we first examine the effects of different per-slot algorithms in our proposed framework. Specifically, we consider two such algorithms:  Deterministic Double Greedy, and Randomized Double Greedy. Note that here Exhaustive Search does not work as we simulate 30 data centers. We evaluate their performance using four metrics: convergence speed, test accuracy (average over last 20 times slots) average utility (denoted as $\frac{1}{T}\sum_t^T U^t$), and average carbon footprint.

The results, as depicted in Fig. \ref{fig:per-slot}, reveal that Randomized Double Greedy slightly outperforms Deterministic Double Greedy. Therefore, we opt to use Randomized Double Greedy as the per-slot solver in our proposed method for subsequent simulations.
% \begin{figure}[h]
% \centering
% \subfigure[Convergence]{
%   \includegraphics[width=0.475\linewidth, height=0.313\linewidth]{figures/per_slot_slotion_convergence.pdf}
%   \label{fig:1sub1}}
% \subfigure[Utility and Carbon Footprint]{
%   \includegraphics[width=0.475\linewidth, height=0.313\linewidth]{figures/per_slot_slotion_bar.pdf}
%   \label{fig:1sub2}}
% \caption{Impact of per-slot algorithms.}
% \label{fig:per-slot}
% \end{figure}

\begin{figure}[h]
\centering
\vspace{-10pt}
\subfigure[Convergence]{
  \includegraphics[width=0.475\linewidth, height=0.313\linewidth]{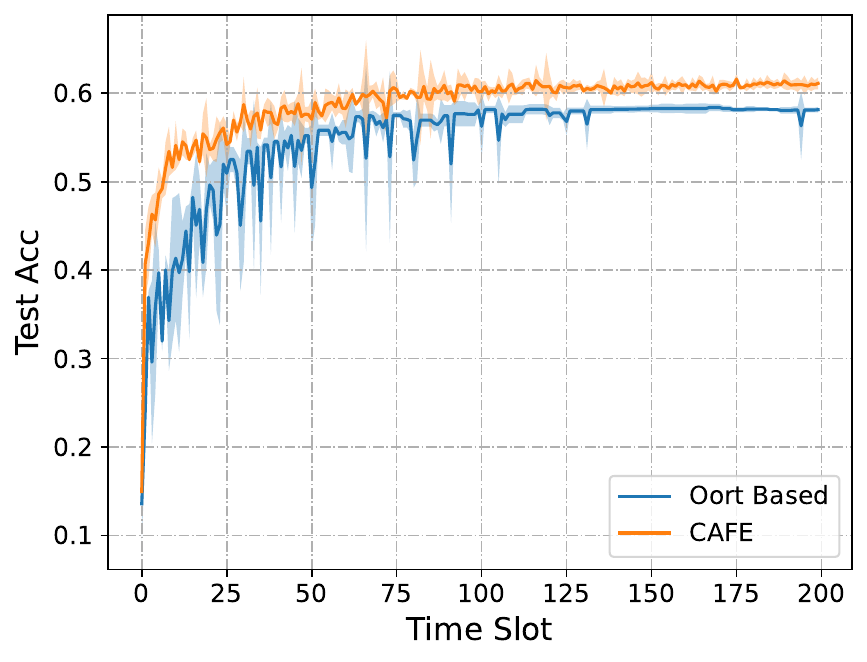}
  \label{fig:2sub3}}
\subfigure[Accuracy and Carbon Footprint]{
  \includegraphics[width=0.475\linewidth, height=0.313\linewidth]{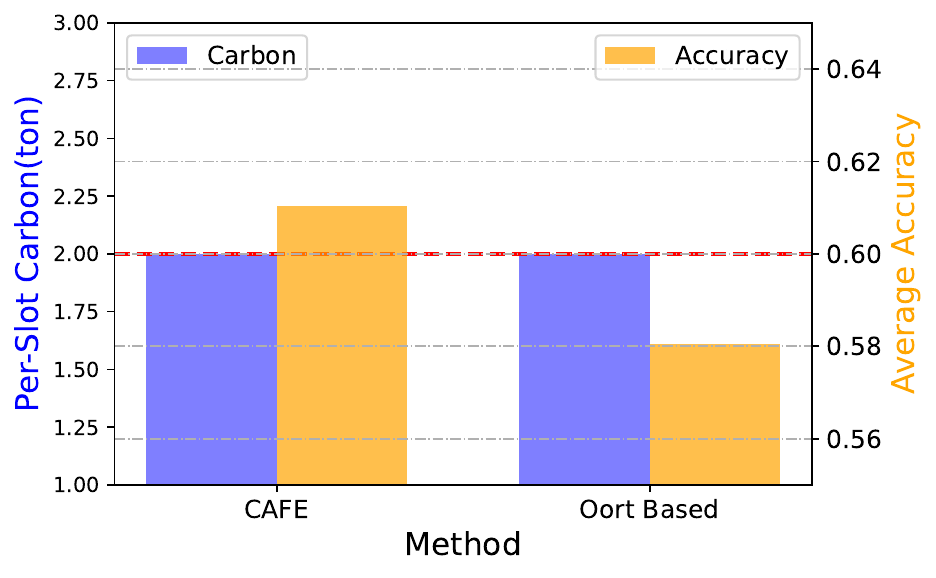}
  \label{fig:2sub4}}
\vspace{-5 pt}
\caption{Performance Comparison with Oort Based Method.}
\label{fig:baseline_loss}
\vspace{-10 pt}
\end{figure}

\begin{figure*}[tt]
	\centering
    \subfigure[Estonia]{
  \includegraphics[width=0.235\linewidth]{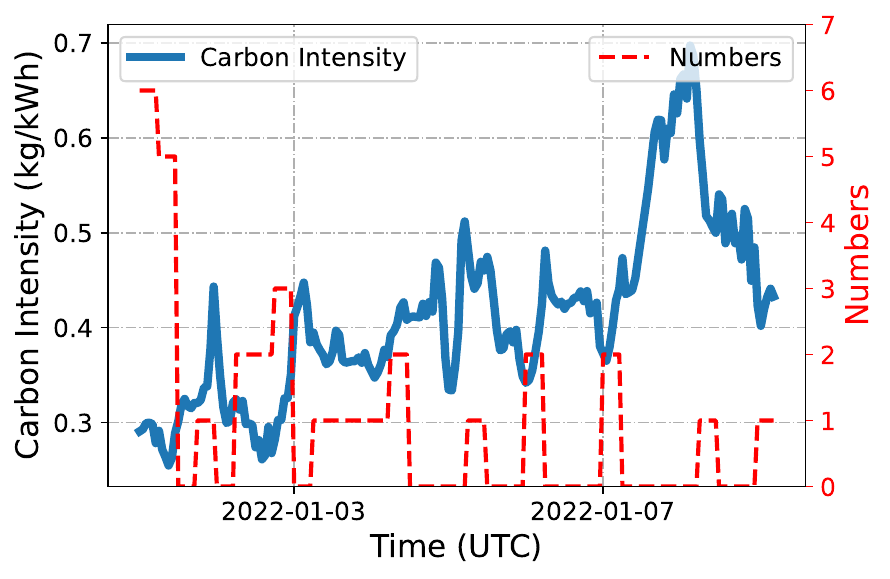}
  \label{fig:estonia}}
  \subfigure[Japan]{
  \includegraphics[width=0.235\linewidth]{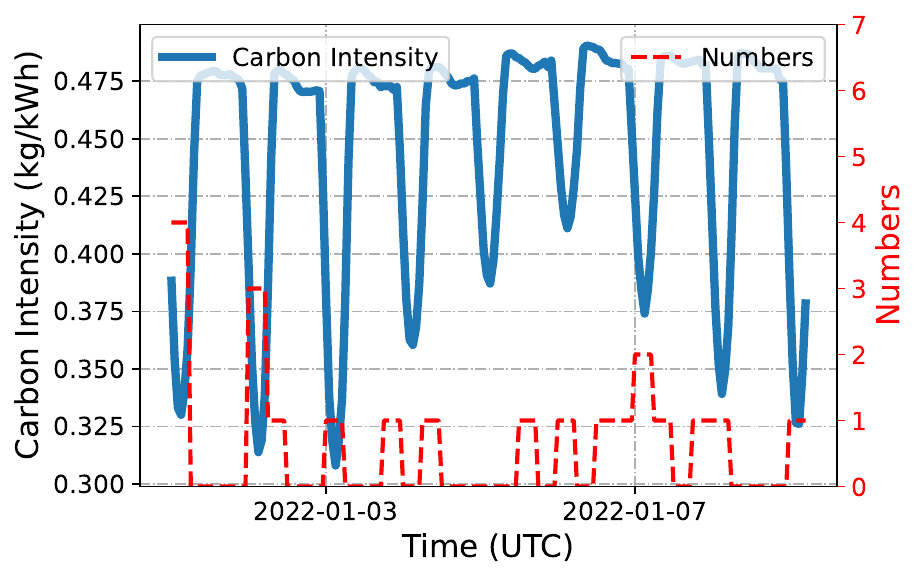}
  \label{fig:japan}}
    \subfigure[US-Texas]{
  \includegraphics[width=0.235\linewidth]{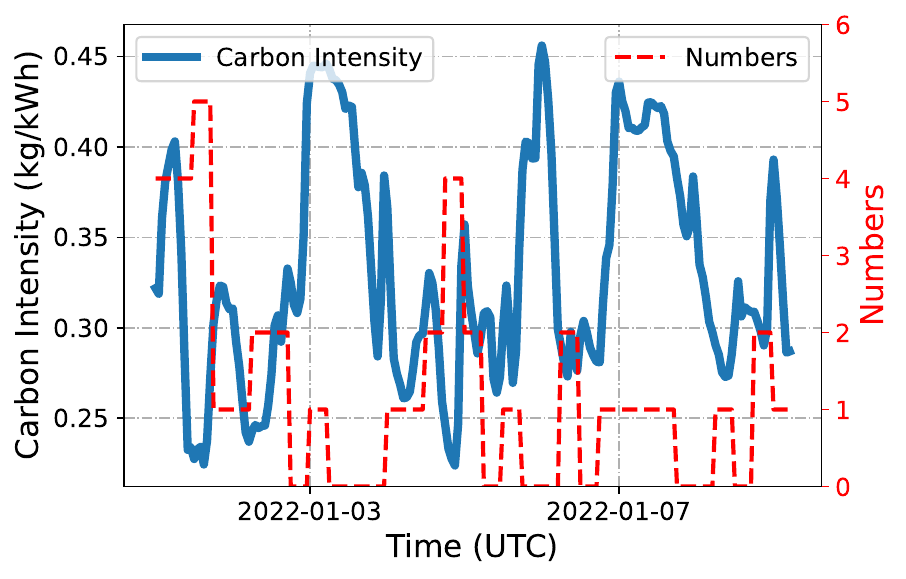}
  \label{fig:texas}}
    \subfigure[Turkey]{
  \includegraphics[width=0.235\linewidth]{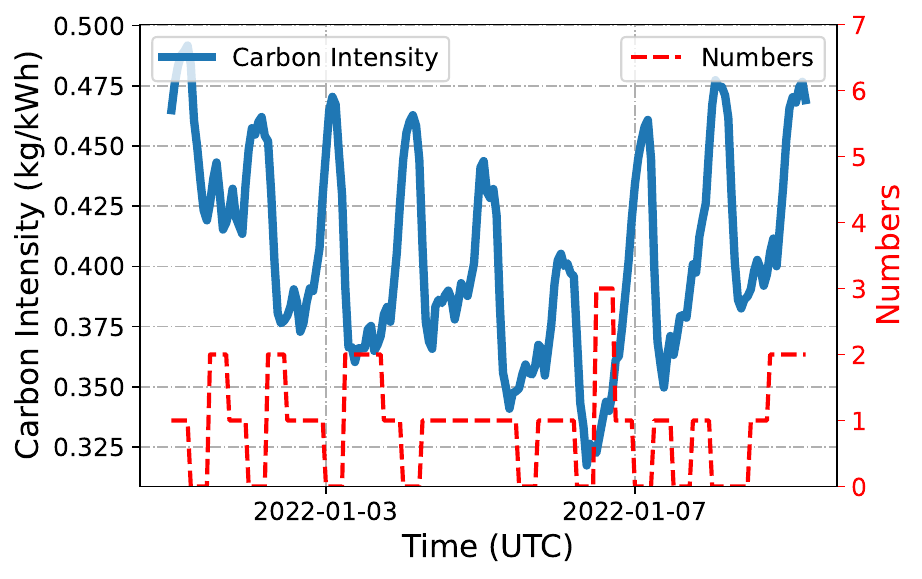}
  \label{fig:turkey}}
  \vspace{-10pt}
  \caption{Total Number of Times the Data Center is Selected in Each 6-hour Interval.}
    \label{fig:carbon_number}
    \vspace{-9pt}
\end{figure*}

\begin{figure*}[tt]
	\centering
    \begin{minipage}[b]{0.74\textwidth}
    \subfigure[Convergence]{
  \includegraphics[width=0.32\linewidth, height=0.21\linewidth]{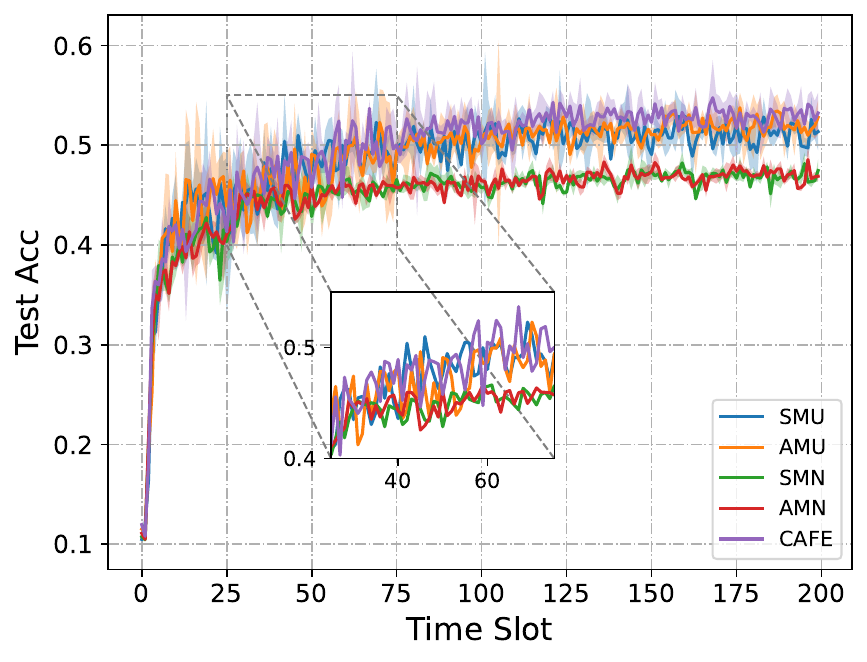}
  \label{fig:3sub1}}
    \subfigure[Utility and Carbon Footprint]{
  \includegraphics[width=0.32\linewidth, height=0.21\linewidth]{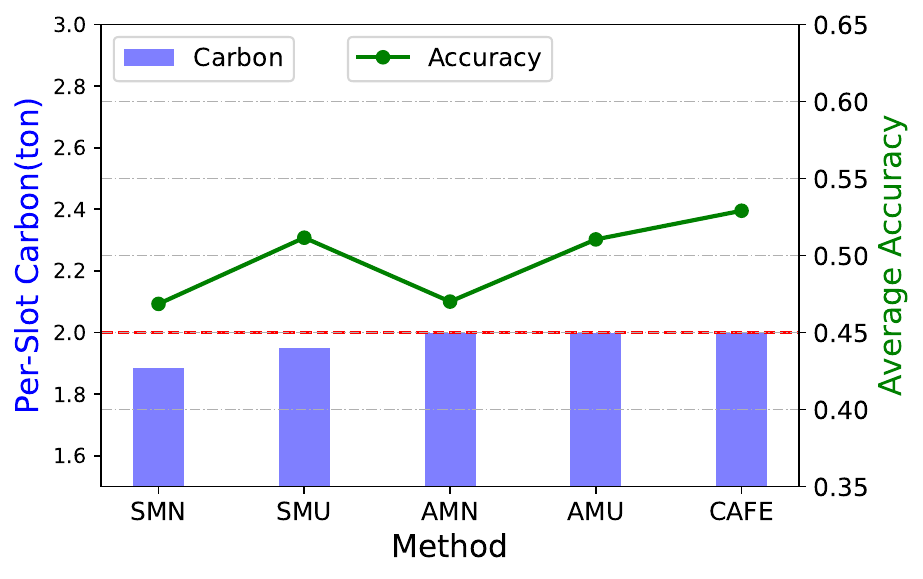}
  \label{fig:3sub2}}
	\subfigure[Participation Times]{
  \includegraphics[width=0.32\linewidth, height=0.21\linewidth]{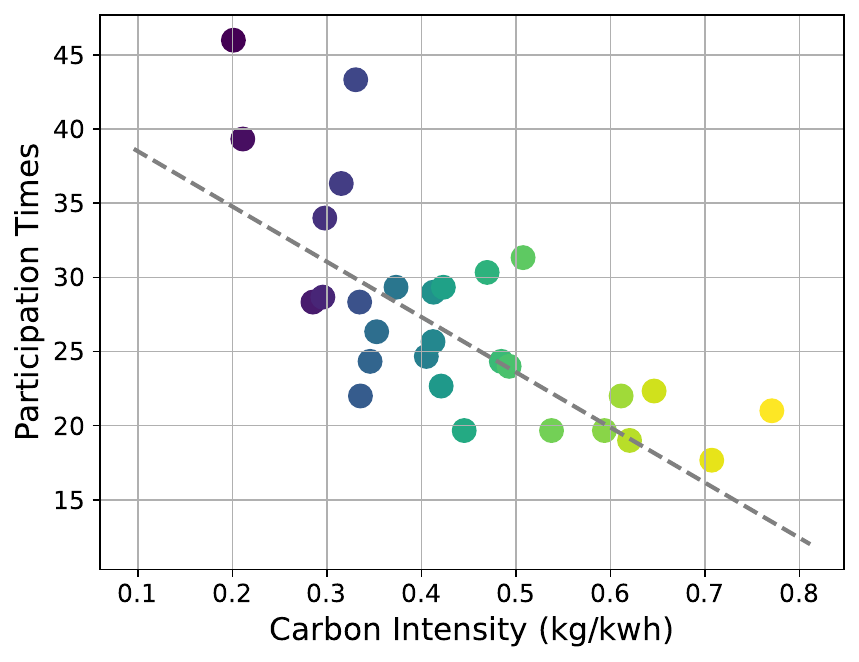}
  \label{fig:3sub3}}
   \vspace{-9pt}
  \caption{Comparsions under IID setting.}
   \vspace{-9pt}
    \label{fig:iid}
	\end{minipage}
	\begin{minipage}[b]{0.2368\textwidth}
	%\hspace{-2mm}
   	 	\includegraphics[width=0.95\linewidth, height=0.65\linewidth]{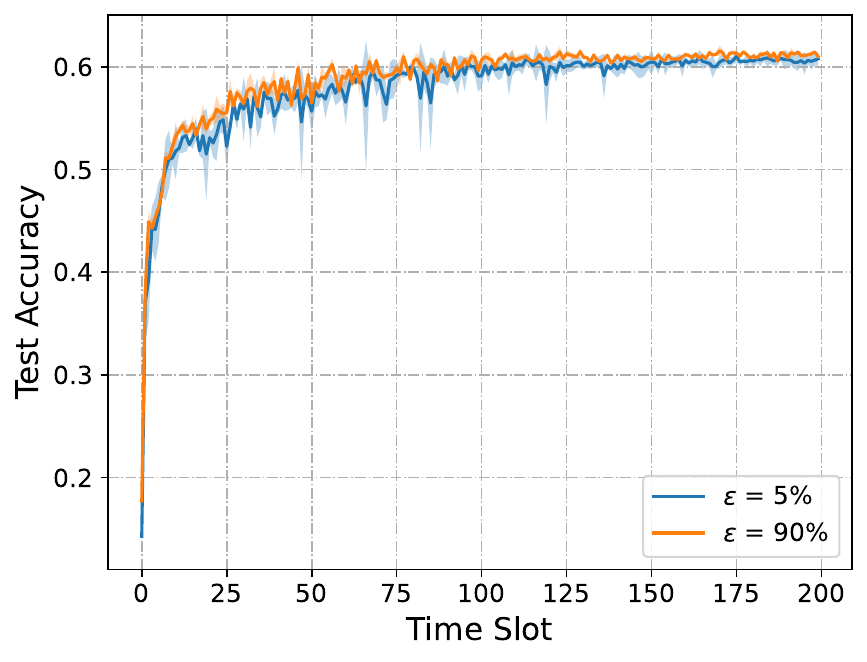} % Modified dimensions
       \vspace{9pt}
   	\caption{Impact of $\epsilon$.}
   	\label{fig:epsilon}
     \vspace{-9pt}
	\end{minipage}
\end{figure*}

\subsubsection{Performance Comparison}
We first compare the performance of our proposed algorithm with the baseline methods on the CIFAR-10 dataset, using default configurations. For the baseline algorithms, we utilize the Greedy Algorithm as the per-slot solver. However, our proposed method employs the Randomized Double Greedy approach. It is important to note that identifying an efficient solver for the baseline methods is not our primary objective in this study. Figure \ref{fig:baseline} offers a comparative analysis, emphasizing key metrics such as convergence performance, average test accuracy, and average carbon footprint. The red line in the figure represents the carbon footprint budget. Our findings indicate that our proposed method provides the best average test accuracy and remains within the carbon budget constraints. Further comparison with the baseline methods reveals a clear advantage for algorithms that use the utility (i.e., $ U^t $) as the objective function over those that focus on participant numbers (i.e., $ K^t $). This observation reaffirms our rationale for using $ U^t $ as a metric to quantify learning performance. Among the baseline methods, those employing a static myopic approach (SMU and SMN) fail to fully utilize the available carbon footprint budget, resulting in suboptimal performance. AMU, which uses utility as its objective and adaptively adjusts the per-slot budget, manages to use the entire carbon allocation while achieving test accuracy comparable to our method. However, as based on Figure \ref{fig:baseline}(a), AMU's convergence rate is slower. This drawback can be largely attributed to its conservative carbon emissions during the initial time slots, which adversely impact its convergence speed. Specifically, AMU's cautious strategy of selecting fewer data centers in the early stages leads to a markedly slower initial rate of convergence compared to our proposed method. Considering that data centers may need to employ the model during training slots, this slower rate limits AMU and underscores the strengths of our proposed approach.

Additionally, we conducted experiments on the CIFAR-100 dataset. The results, illustrated in Fig. \ref{fig:baseline_cifar100}, indicate a performance comparable to that observed with the CIFAR-10 dataset, wherein our proposed method surpasses the baseline approaches.

\subsubsection{Utility Function Efficiency}

In our proposed method, the concept of coreset gradient is employed to define our objective. It is important to acknowledge that within Federated Learning, alternative criteria for client selection exist. Our approach opts for the incorporation of the strategy described in \cite{lai2021oort}, recognized as a leading client selection technique, to serve as a benchmark for comparison. The outcomes, shown in Fig. \ref{fig:baseline_loss}, demonstrate that the application of the coreset gradient, as implemented in our method, significantly surpasses the performance of the OORT-based method, even when both strategies operate within an equivalent carbon budget.

\subsubsection{Carbon Intensity and Number of Times Selected}
We examine the correlation between a data center's carbon intensity and its likelihood of being selected for participation. For clarity, we plot the average number of times a data center is chosen over each 6-hour period (with a maximum selection frequency of 6 and a minimum of 0). The results, depicted in Fig. \ref{fig:carbon_number}, focus on four random data centers. They indicate a negative correlation: as a data center's carbon intensity decreases, its selection frequency increases. However, some anomalies exist where high carbon intensity corresponds with high selection frequency or vice versa. This can be attributed to our selection criteria, which consider relative carbon intensity. For instance, even if a data center experiences a period of low carbon intensity, another center might have an even lower intensity, precluding the selection of the former. Moreover, our selection is not solely based on carbon intensity; it also considers learning performance which can be impacted by each data center's local data distribution. These factors—non-IID distribution and relative carbon intensity—account for the observed discrepancies.

\subsubsection{IID Setting}
In the simulations above, we focused on scenarios where each data center stores non-IID data. In this section, we explore cases where the data stored at each data center is IID distributed. As shown in Fig. \ref{fig:iid} (a), (b), our proposed method continues to outperform both baseline methods. However, it is worth noting that the gains in test accuracy and convergence speed are less pronounced when compared to the results depicted in Fig. \ref{fig:baseline}. This diminished impact is due to the fact that when data across different data centers is IID, the importance of selecting a `coreset' data center is reduced.

Furthermore, we present the average carbon intensity and participation frequency of each data center in Fig.  \ref{fig:iid} (c). The IID setting allows us to largely ignore the impact of data distribution, focusing solely on carbon intensity. The results reveal that Denmark, which has the lowest average carbon intensity, participates the most, while data centers with higher average carbon intensity are rarely selected to participate. 

% We acknowledge that this could raise concerns about `equity' and aim to address this issue in our future work.

\subsubsection{Impact of $\epsilon$}
In Figure \ref{fig:epsilon}, we show the influence of the parameter $ \epsilon $ on the quality of data center selection for participation. The results confirm that even a small $ \epsilon $ value can offer a reliable estimate of full local data training in our simulations. However, it is important to acknowledge that the appropriate setting for $ \epsilon $ may vary depending on the specific learning task.

\subsubsection{Impact of Budget}
In Fig. \ref{fig:budget}, we investigate the impact of the carbon footprint budget $H$ on both our proposed method and the baseline methods. Intuitively, as the budget increases, the average test accuracy for all methods also rises. This is expected, as a larger carbon budget enables more data centers to participate in training. Importantly, our proposed method continues to outperform the two baseline methods across various budget levels. However, it is worth noting that the advantage of our proposed method diminishes as the budget expands. This decrease in relative benefit is logical: as more data centers are allowed to participate in each time slot, the task of selecting a `coreset' becomes less challenging, thereby reducing the edge of our proposed approach.

\begin{figure}[h]
\centering
\vspace{-10 pt}
\subfigure[Average Carbon Footprint]{
  \includegraphics[width=0.475\linewidth, height=0.313\linewidth]{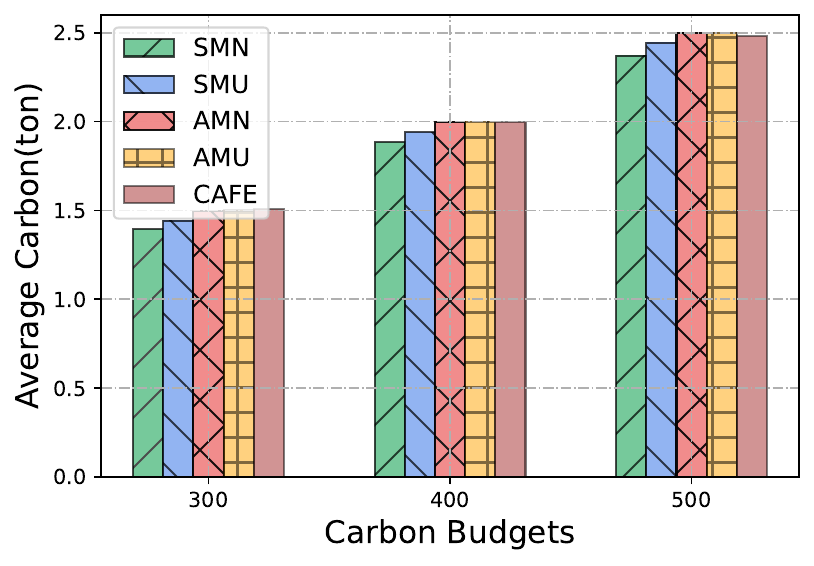}
  \label{fig:4sub1}}
\subfigure[Average Accuracy]{
  \includegraphics[width=0.475\linewidth, height=0.313\linewidth]{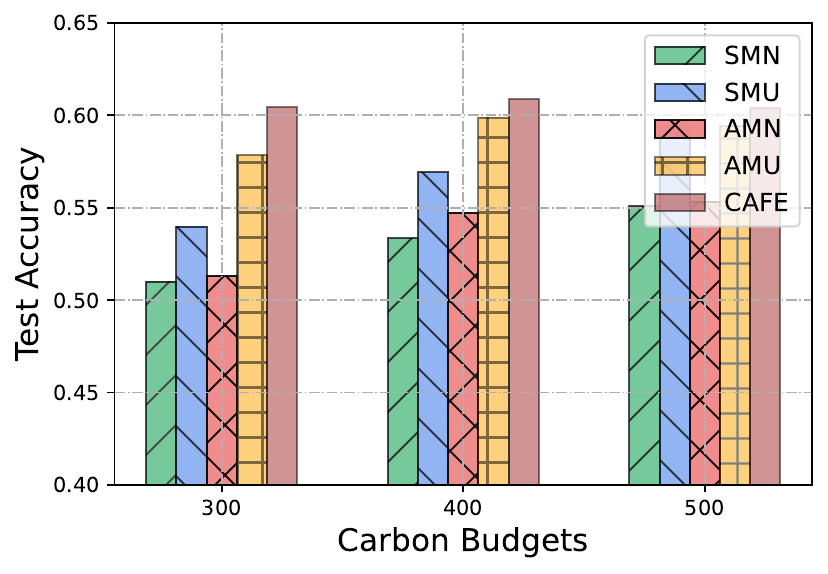}
  \label{fig:4sub2}}
  \vspace{-8 pt}
\caption{Impact of Budget.}
\label{fig:budget}
\vspace{-5 pt}
\end{figure}

\subsubsection{Comparisons with Extreme Cases}
In this subsection, we assess the performance of our proposed method, CAFE, by comparing it to two extreme cases.

In the first extreme case, the sole focus is on maximizing utility, with no consideration given to carbon budget constraints. Specifically, in Fig. \ref{fig:usub2}, the point labeled $ K^t = 4 $ represents the test accuracy achieved when four data centers are selected in each time slot, aiming solely to maximize utility $ U^t $ without taking any carbon constraints into account. Fig. \ref{fig:usub2} demonstrates that under identical carbon emissions, our proposed method consistently outperforms the scenario where the focus is solely on utility.

In the second extreme case, the exclusive focus is on minimizing carbon emissions, without consideration of utility. The results shown in Fig. \ref{fig:csub1} indicate that for any $K^t$, our proposed method continues to outperform the case that solely prioritizes carbon reduction under the same corresponding carbon emission. These findings from both extreme cases underscore the necessity of considering both utility and carbon footprint concurrently, rather than treating them as isolated factors.

\begin{figure}[h]
\centering
\subfigure[Focus on Utility Only]{
  \includegraphics[width=0.475\linewidth, height=0.313\linewidth]{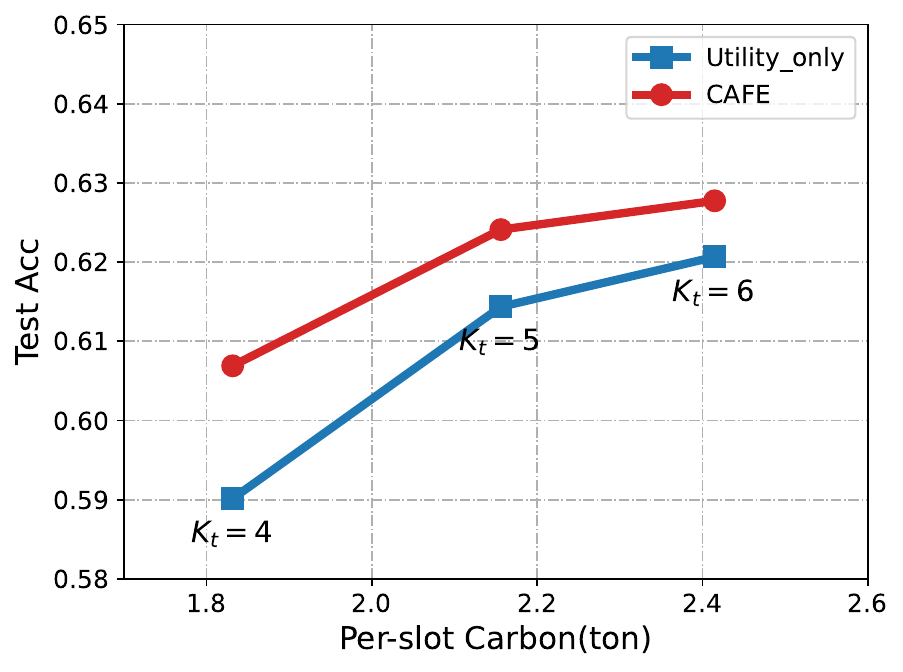}
  \label{fig:usub2}}
\subfigure[Focus on Carbon Only]{
  \includegraphics[width=0.475\linewidth, height=0.313\linewidth]{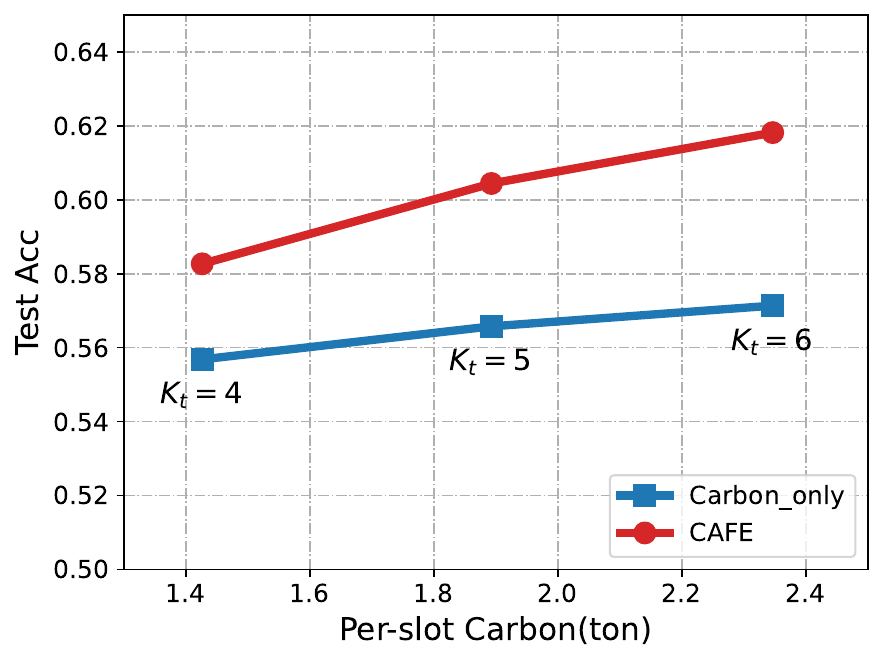}
  \label{fig:csub1}}
  \vspace{-8 pt}
\caption{Comparisons with Extreme Cases}
\label{fig:U/C_only}
\vspace{-5 pt}
\end{figure}

\subsubsection{Impact of Algorithms Parameters}
Lastly, we investigate the impact of the algorithm parameters of the proposed method on its performance.

\textbf{Impact of control parameter $V$.}
The control parameter $ V $ is instrumental in balancing the dual objectives of maximizing utility and adhering to the carbon footprint budget. A larger $ V $ places a greater emphasis on utility maximization, while a smaller $ V $ is more geared towards staying within the carbon budget. As illustrated in Fig. \ref{fig:V}, adjusting $ V $ effectively manages this trade-off. When $ V $ is increased, the proposed method achieves better test accuracy. However, it is important to note that a larger $ V $ also results in greater deviations from the carbon budget constraint. This suggests that focusing too heavily on performance optimization could compromise adherence to carbon limits, a conclusion that aligns with our theoretical findings. Therefore, the selection of $ V $ should be customized based on the specific needs and constraints of the application at hand.

% \begin{figure}[h]
% \centering
% \begin{minipage}[b]{0.47\textwidth}
% 	%\hspace{-2mm}
%    	\includegraphics[width=0.45\linewidth]{figures/V_lines.pdf}
%         \caption{Impact of V.}
%     \label{fig:V}
% 	\end{minipage}
% \begin{minipage}[b]{0.47\textwidth}
% 	%\hspace{-2mm}
%    	\includegraphics[width=0.45\linewidth]{figures/W_lines.pdf}
%         \caption{Impact of $q^0$.}
%     \label{fig:q0}
% 	\end{minipage}
% \end{figure}

\begin{figure}[h]
\vspace{-5 pt}
    \centering
    \begin{minipage}[t]{0.49\linewidth}
        \includegraphics[width=0.99\linewidth]{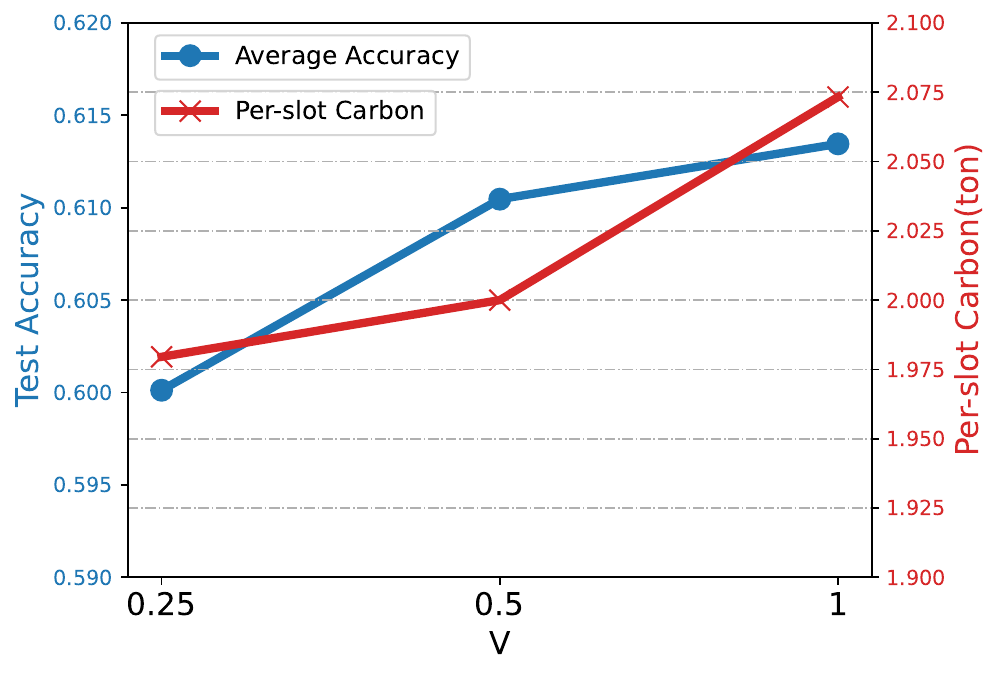}
        \vspace{-20 pt}
        \caption{Impact of V.}
	    \label{fig:V}
    \end{minipage}
     \begin{minipage}[t]{0.49\linewidth}
        \includegraphics[width=0.99\linewidth]{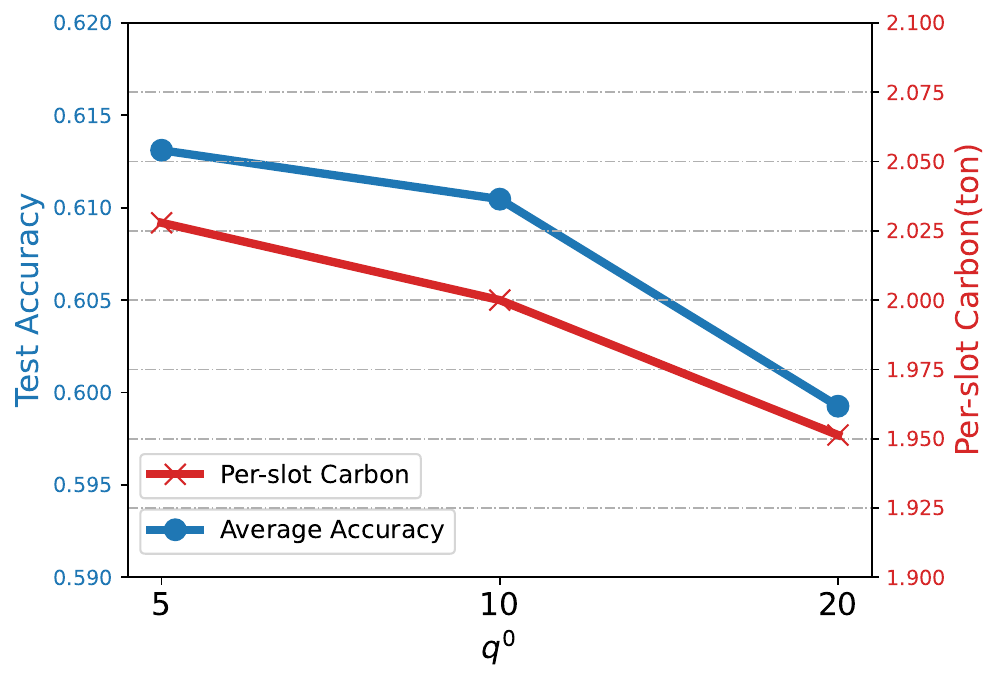}
                \vspace{-20 pt}
        \caption{Impact of $ q^0 $.}
	    \label{fig:q0}
    \end{minipage}
    \vspace{-5 pt}
\end{figure}

% \begin{figure}[h]
% \centering
% \subfigure[Test accuracy and Carbon Footprint]{
%   \includegraphics[width=0.475\linewidth, height=0.313\linewidth]{figures/W_lines.pdf}
%   \label{fig:7sub2}}
% \subfigure[Convergence]{
%   \includegraphics[width=0.475\linewidth, height=0.313\linewidth]{figures/W_convergence.pdf}
%   \label{fig:7sub1}}
% \caption{Impact of W.}
% \label{fig:W}
% \end{figure}

\textbf{Impact of initial virtual queue length $q^0$.}
Figure~\ref{fig:q0} explores the impact of the initial virtual queue value $q^0$ on test accuracy and carbon emissions. The simulation results corroborate our theoretical analysis, showing that a larger $ q^0 $ tends to reduce carbon usage by limiting the number of data centers participating in the initial time slots. However, setting $ q^0 $ too high can adversely affect the average utility, which is further reflected in lower test accuracy. Contrary to the common practice of setting $ q^0 = 0 $, our findings suggest that a relatively small $ q^0 $ is effective in minimizing carbon emissions while still maintaining good convergence performance.

\section{Conclusion}

Given the growing imperative to harmonize the computational requirements of AI training with environmental sustainability, our research elucidates the complexities of training AI models in geo-distributed data centers. We introduce the Carbon-Aware Federated Learning (CAFE) framework, demonstrating its potential to enhance learning outcomes while adhering to environmental constraints. Building on the merits of Federated Learning, CAFE not only elevates learning performance and data privacy but also adeptly manages the variable nature of carbon intensity. Our empirical simulations validate CAFE's efficacy, illustrating its superiority over conventional methods in achieving a balanced learning approach without jeopardizing environmental sustainability. As the digital era progresses, the integration of such green AI strategies becomes essential for ensuring a sustainable technological and environmental trajectory.

\bibliographystyle{ACM-Reference-Format}
\bibliography{reference}

\appendix
\section{Proof of Proposition \ref{prop3}}
\label{proof_p1}
\begin{proof}
    We start with the definition of utility defined in Eq. \ref{metric}.
    \begin{align}
         & |U^t(\boldsymbol{a^t}|w_1) - U^t(\boldsymbol{a^t}|w_2)| \nonumber\\
         = &  |\sum_{j \in \mathcal{N}} \min_{i \in \mathcal{K}^t} \|\nabla f_j(w_2) - \nabla f_i(w_2)\| -  \|\nabla f_j(w_1) - \nabla f_i(w_1)\|| \nonumber\\
         \leq &  \sum_{j \in \mathcal{N}} \min_{i \in \mathcal{K}^t} \|\nabla f_j(w_2) - \nabla f_i(w_2)\|  +  \|\nabla f_j(w_1) - \nabla f_i(w_1)\| \nonumber\\
         = &  \sum_{j \in \mathcal{N}} \min_{i \in \mathcal{K}^t}   \|\nabla f_j(w_2) - \nabla f(w_2) + \nabla f(w_2) - \nabla f_i(w_2)\| \nonumber\\
          +&   \sum_{j \in \mathcal{N}} \min_{i \in \mathcal{K}^t}   \|\nabla f_j(w_1) - \nabla f(w_1) + \nabla f(w_1) - \nabla f_i(w_1)\|  \nonumber\\
          \leq &  \sum_{j \in \mathcal{N}} \min_{i \in \mathcal{K}^t}  ( \|\nabla f_j(w_2) - \nabla f(w_2)\| + \|\nabla f(w_2) - \nabla f_i(w_2)\|) \nonumber\\
          +&   \sum_{j \in \mathcal{N}} \min_{i \in \mathcal{K}^t}   (\|\nabla f_j(w_1) - \nabla f(w_1)\| + \|\nabla f(w_1) - \nabla f_i(w_1)\| )\nonumber\\
          \leq &  \sum_{j \in \mathcal{N}} 2\delta + \sum_{j \in \mathcal{N}} 2\delta = 4 N \delta
    \end{align}
    where the first and second inequality is based on triangle inequality and the last inequality is based on Assumption 3.
\end{proof}

\section{Proof of Theorem \ref{thm1}}
\label{proof_t1}
\begin{proof}
The Lyapunov drift of the virtual queue $q^t$  is:
\begin{align}
    & \Delta (t) \coloneqq \frac{1}{2} \left[ (q^{t+1})^2-(q^{t})^2 \right] \leq \frac{1}{2} \left[\left( {c^t - H/T + q^{t}} \right)^2-\left( q^{t} \right)^2\right] \nonumber\\
    & = q^t(c^t - H/T) + \frac{1}{2} (c^t - H/T)^2  \leq q^t(c^t - H/T) + B_1
\end{align}
where $B_1 > 0$ is a finite constant number since the total carbon constraint is bounded. Then we have:
\begin{align}
% \label{eq_1}
    & V \cdot U^t (\boldsymbol{a}^t) - \Delta(t)
     \geq  V \cdot U^t (\boldsymbol{a}^t) - q^t(c^t - H/T) - B_1 \label{eq_B1_constrain}
\end{align}
Note $\boldsymbol{a}^{*,0}, \dots, \boldsymbol{a}^{*,t}$ be the sequence of data center selection decisions derived by our per-slot algorithm. The above inequality clearly show that such sequence decisions at least be $1/\gamma$-approximate to the optimal solutions which maximizes the lower bound of $V \cdot U^t (\boldsymbol{a}^t) - \Delta(t)$. Next, we consider a specific sequence of decisions where $\Bar{a}_i^t = 0, \forall i,t$. In this case, the utility metric becomes $U^t(\boldsymbol{\Bar{a}}^t) =  b - 2 N \max_{i \in \mathcal{N}} \| \nabla f_i(w)\| $ and $c^t (\boldsymbol{\Bar{a}}^t| \boldsymbol{\beta}^t) = \sum_{i \in \mathcal{N}} \beta_i^t E_i^c$. Consider Eq. \eqref{eq_B1_constrain}, we have:
\begin{align}
    &V \cdot U^t (\boldsymbol{a}^{*,t}) - \Delta(t)
    \geq   V \cdot U^t (\boldsymbol{a}^{*,t}) - q^t(c^t - H/T) - B_1\nonumber \\
    \geq &\frac{1}{\gamma}V( b -  2 N \max_{i \in \mathcal{N}} \| \nabla f_i(w)\| ) - \frac{1}{\gamma}q^t \sum_{i \in \mathcal{N}} \beta_i^t E_i^c + q^tH/T- B_1 \nonumber\\
     \geq &\frac{1}{\gamma}V( b -  2 N \max_{i \in \mathcal{N}} \| \nabla f_i(w)\|) - \frac{1}{\gamma}q^t (\sum_{i \in \mathcal{N}} \beta_i^t E_i^c - H/T) - B_1 \nonumber\\
     \geq &\frac{1}{\gamma}V( b -  2 N \max_{i \in \mathcal{N}} \| \nabla f_i(w)\| ) - B_1
\end{align}
where the second inequality is based on Definition 1 and the last inequality is based on the Assumption 1. 
We arrange the above inequality and then have:
\begin{align}
    \Delta(t) &\leq V \cdot U^t (\boldsymbol{a}^{*,t}) - \frac{1}{\gamma}V( b -  2 N \max_{i \in \mathcal{N}} \| \nabla f_i(w)\|) + B_1 \nonumber\\
    &\leq \frac{\gamma - 1}{\gamma} Vb + \frac{2VN}{\gamma}  \max_{i \in \mathcal{N}} \| \nabla f_i(w)\| + B_1 \nonumber\\
    &\leq \frac{V}{\gamma} ((\gamma -1)b + 2N G) + B_1 \nonumber\\
\end{align}
The last inequality holds based on the Assumption \ref{ass2}, Then we further note that
\begin{align}
     \frac{1}{2} \left[ (q^{T})^2-(q^{0})^2 \right] = \sum_{t=0}^{T-1} \Delta(t) \leq T (\frac{V}{\gamma} ((\gamma -1)b + 2N G) + B_1 )
\end{align}

Then we have that $q^T \leq \sqrt{(q^0)^2 + T (\frac{2V}{\gamma}((\gamma -1)b + 2N G) + 2B_1)}$. According to the virtual queue dynamics,  $q^{t+1} \geq q^t + c^t - H/T$. This leads to 
\begin{align}
    &\frac{1}{T}\sum_{t=0}^{T-1} c^t - \frac{H}{T} \leq \frac{1}{T}\sum_{t=0}^{T-1}(q^{t+1} - q^t) = \frac{q^T - q^0}{T}\nonumber\\
    \leq & \sqrt{\frac{(q^0)^2}{T^2} + \frac{\frac{2V}{\gamma}((\gamma -1)b + 2N G) + 2B_1}{T}} - \frac{q^0}{T}
\end{align}

\end{proof}

\section{Proof of Theorem \ref{thm2}}
\label{proof_t2}
\begin{proof}

In the following proof, we note that $\boldsymbol{a}^{opt,t} $ represents the data center selection outcome at each time slot $t$ that achieves the offline optimal solution. Additionally, as denoted in the before section,  $w^{full,t} $ signifies the model weight, assuming that \textbf{all} data center being selected at time slot $t$. In contrast, $w^{real, t} $ designates the actual model weight attained at time slot $t$. This distinction serves to highlight the relationship between model weights and data center selections across different time slots, adding a layer of complexity to the problem.

Consider the objective function minus drift of the per-slot algorithm:
\begin{align}
    & V \cdot \frac{1}{T}\sum_{t=0}^{T-1} U^t(\boldsymbol{a^t}|w^{real,t}) - \frac{1}{T}\sum_{t=0}^{T-1} \Delta(t)\nonumber\\
    \geq & V \cdot \frac{1}{T}\sum_{t=0}^{T-1} U^t(\boldsymbol{a^t}|w^{real,t})- \frac{1}{T}\sum_{t=0}^{T-1} q^tc^t + \frac{1}{T}\sum_{t=0}^{T-1}q^t H/T - B_1\nonumber\\
    \geq &  \frac{1}{T}\sum_{t=0}^{T-1} (V \cdot U^t(\boldsymbol{a^t}|w^{real,t})- q^tc^t) + \frac{1}{T}\sum_{t=0}^{T-1}q^t H/T - B_1\nonumber\\
    \geq &  \frac{1}{\gamma T}\sum_{t=0}^{T-1}(V \cdot U^t(\boldsymbol{a}^{opt, t}|w^{real,t})- q^tc^{opt,t}) + \frac{1}{T}\sum_{t=0}^{T-1}q^t H/T - B_1\nonumber\\
    \geq & \frac{1}{\gamma T}V \cdot \sum_{t=0}^{T-1}U^t(\boldsymbol{a}^{opt, t}|w^{real,t}) - \frac{1}{\gamma T}\sum_{t=0}^{T-1}q^t(c^{opt,t} - H/T) - B_1 \nonumber\\
    \geq & \frac{1}{\gamma T}V\sum_{t=0}^{T-1}[U^t(\boldsymbol{a}^{opt, t}|w^{full,t}) - 4N\delta]  - \frac{1}{\gamma T}\sum_{t=0}^{T-1}q^t(c^{opt,t} - H/T)  - B_1 \nonumber\\
    \geq &  \frac{V \cdot \text{OPT}}{\gamma}- \frac{4}{\gamma}NV\delta - \frac{1}{\gamma T}\sum_{t=0}^{T-1}q^t(c^{opt,t} - H/T)  - B_1 \nonumber\\
    \geq &  \frac{V \cdot \text{OPT}}{\gamma}- \frac{4}{\gamma}NV\delta - \frac{1}{\gamma T}\sum_{t=0}^{T-1}q^t|c^{opt,t} - H/T|  - B_1 
\end{align}
The first inequality is due to the bound on $\Delta(t)$, the forth inequality is due to Proposition \ref{prop3} and the last inequality is based on the fact that $q^t$ is non-negative for all $t$. Here, $c^{opt,t}$ is the cost incurred in time slot $t$ by the offline optimal solution.  

Note that we can bound $q^t$ by the definition:
\begin{align}
    q^t - q^0 \leq t (c^{max} - H/T) \leq Tc^{max} - H
\end{align}
where  $c^{max}$ denotes the $max\{c^{opt,t}\}$.

Then we substitute and have:
\begin{align}
    & V \cdot \frac{1}{T}\sum_{t=0}^{T-1} U^t(\boldsymbol{a^t})- \frac{1}{T}\sum_{t=0}^{T-1} \Delta(t) \nonumber\\
    \geq & \frac{V \cdot \text{OPT}}{\gamma}- \frac{4}{\gamma}NV\delta - \frac{1}{\gamma T}\sum_{t=0}^{T-1}q^t|c^{opt,t} - H/T|  - B_1 \nonumber\\
    \geq &  \frac{V \cdot \text{OPT}}{\gamma}- \frac{4}{\gamma}NV\delta - \frac{1}{\gamma T}\sum_{t=0}^{T-1}q^0|c^{opt,t} - H/T|  \nonumber\\
    & - \frac{1}{\gamma T}\sum_{t=0}^{T-1}(T c^{max} - H)|c^{opt,t} - H/T| - B_1 \nonumber\\
    \geq &  \frac{V \cdot \text{OPT}}{\gamma}- \frac{4}{\gamma}NV\delta - \frac{q^0c^{max}}{\gamma} - \frac{(T c^{max} - H)c^{max}}{\gamma} - B_1
\end{align}

Finally, noticing $\frac{1}{T}\sum_{t=0}^{T-1} \Delta(t) \geq -\frac{1}{2T}(q^0)^2$ and moving it to the right hand side yields
\begin{align}
&V\cdot \frac{1}{T}\sum_{t=0}^{T-1} U^t(\boldsymbol{a^t}) \geq  \frac{V\text{OPT}}{\gamma}- \frac{4}{\gamma}NV\delta - (\frac{q^0c^{max}}{\gamma} + \frac{1}{2T}(q^0)^2) \nonumber\\
&- \frac{(T c^{max} - H)c^{max}}{\gamma} - B_1
\end{align}
Dividing both sides by $V$ yields:
\begin{align}
\label{thm2_result}
& \frac{1}{T}\sum_{t=0}^{T-1} U^t(\boldsymbol{a^t}) \geq  \frac{\text{OPT}}{\gamma}- \underbrace{\frac{4}{\gamma}N\delta}_{G_1} - \underbrace{\frac{1}{V}(\frac{q^0c^{max}}{\gamma} + \frac{1}{2T}(q^0)^2)}_{G_2} \nonumber\\
&- \underbrace{\frac{1}{V}(\frac{(T c^{max} - H)c^{max}}{\gamma} + B_1)}_{G_3}
\end{align}
\end{proof}

\end{document}